\documentclass[10pt,journal,compsoc]{IEEEtran}
\ifCLASSOPTIONcompsoc
  \usepackage[nocompress]{cite}
\else
  \usepackage{cite}
\fi
\ifCLASSINFOpdf
\else
\fi
\usepackage{amssymb,amsmath,amsthm}
\usepackage{algorithmic}
\usepackage{algorithm}

\newtheorem{theorem}{Theorem}
\newtheorem{lemma}[theorem]{Lemma}

\usepackage{footnote}
\usepackage{graphicx}

\newcommand{\eg}{\emph{e.g. ,}}
\newcommand{\ie}{\emph{i.e. ,}}
\newcommand{\etal}{\emph{et.al.}}

\begin{document}
%


\title{Universal Adder Neural Networks}
%
%
%
%

\author{Hanting Chen, Yunhe Wang,~\IEEEmembership{Member,~IEEE}, Chang Xu,~\IEEEmembership{Member,~IEEE},\\Chao Xu, Chunjing Xu, and Tong Zhang,~\IEEEmembership{Fellow,~IEEE}
\IEEEcompsocitemizethanks{
\IEEEcompsocthanksitem Hanting Chen and Chao Xu are with the Key Laboratory of Machine Perception (Ministry of Education) and Coopertative Medianet Innovation Center, School of EECS, Peking University, Beijing 100871, P.R. China. E-mail:htchen@pku.edu.cn, xuchao@cis.pku.edu.cn
\IEEEcompsocthanksitem Hanting Chen, Yunhe Wang and Chunjing Xu are with the Noah's Ark Laboratory, Huawei Technologies Co., Ltd, HuaWei Building, No.3 Xinxi Road, Shang-Di Information Industri Base, Hai-Dian District, Beijing 100085, P.R. China. E-mail:htchen@pku.edu.cn, yunhe.wang@huawei.com, xuchunjing@huawei.com
\IEEEcompsocthanksitem Chang Xu is with the School of Computer Science in the Faculty of Engineering and Information Technologies at The University of Sydney, J12 Cleveland St, Darlington NSW 2008, Australia. E-mail: c.xu@sydney.edu.au
\IEEEcompsocthanksitem Tong Zhang is with the School of Computer Science and Engineering, and Mathematics at the Hong Kong University of Science and Technology, Clear Water Bay, Kowloon, Hong Kong. E-mail: tongzhang@tongzhang-ml.org
\IEEEcompsocthanksitem Hanting Chen and Yunhe Wang contributed equally to this manuscript.
\IEEEcompsocthanksitem Correspondence to Chang Xu.
}}

\IEEEtitleabstractindextext{%
\begin{abstract}

Compared with cheap addition operation, multiplication operation is of much higher computation complexity. The widely-used convolutions in deep neural networks are exactly cross-correlation to measure the similarity between input feature and convolution filters, which involves massive multiplications between float values. In this paper, we present adder networks (AdderNets) to trade these massive multiplications in deep neural networks, especially convolutional neural networks (CNNs), for much cheaper additions to reduce computation costs. In AdderNets, we take the $\ell_1$-norm distance between filters and input feature as the output response. We first develop a theoretical foundation for AdderNets, by showing that both the single hidden layer AdderNet and the width-bounded deep AdderNet with ReLU activation functions are universal function approximators. An approximation bound for AdderNets with a single hidden layer is also presented. We further analyze the influence of this new similarity measure on the optimization of neural network and develop a special training scheme for AdderNets. Based on the gradient magnitude, an adaptive learning rate strategy is proposed to enhance the training procedure of AdderNets.  AdderNets can achieve a 75.7\% Top-1 accuracy and a 92.3\% Top-5 accuracy using ResNet-50 on the ImageNet dataset without any multiplication in the convolutional layer.
\end{abstract}

\begin{IEEEkeywords}
Efficient network, deep learning, computer vision, energy consumption.
\end{IEEEkeywords}}

\maketitle

\IEEEdisplaynontitleabstractindextext

%
\IEEEpeerreviewmaketitle

\section{Introduction}
Given the advent of Graphics Processing Units (GPUs), deep convolutional neural networks (CNNs) with billions of floating number multiplications could receive speed-ups and make important strides in a large variety of computer vision tasks, \eg image classification~\cite{VGG,krizhevsky2012imagenet}, object detection~\cite{ren2015faster}, segmentation~\cite{long2015fully}, and human face verification~\cite{wen2016discriminative}. However, the high-power consumption of these high-end GPU cards (\eg 250W+ for GeForce RTX 2080 Ti) has blocked modern deep learning systems from being deployed on mobile devices, \eg smart phone, camera, and watch. Existing GPU cards are far from svelte and cannot be easily mounted on mobile devices. Though the GPU itself only takes up a small part of the card, we need many other hardware for supports, \eg memory chips, power circuitry, voltage regulators and other controller chips. It is therefore necessary to study efficient deep neural networks that can run with affordable computation resources on mobile devices. 

Addition, subtraction, multiplication and division are the four most basic operations in mathematics. It is widely known that multiplication is slower than addition, but most of the computations in deep neural networks are multiplications between float-valued weights and float-valued activations during the forward inference. There are thus many papers on how to trade multiplications for additions, to speed up deep learning. The seminal work~\cite{courbariaux2015binaryconnect} proposed BinaryConnect to force the network weights to be binary (\eg -1 or 1), so that many multiply-accumulate operations can be replaced by simple accumulations.  After that, Hubara~\etal~\cite{hubara2016binarized} proposed BNNs, which binarized not only weights but also activations in convolutional neural networks at run-time. Moreover, Rastegari~\etal~\cite{rastegari2016xnor} introduced scale factors to approximate convolutions using binary operations and outperform~\cite{hubara2016binarized,rastegari2016xnor} by large margins. Zhou~\etal~\cite{zhou2016dorefa} utilized low bit-width gradient to accelerate the training of binarized networks. Cai~\etal~\cite{cai2017deep} proposed an half-wave Gaussian quantizer for forward approximation, which achieved much closer performance to full precision networks. 

Though binarizing filters of deep neural networks significantly reduces the computation cost, the original recognition accuracy often cannot be preserved. In addition, the training procedure of binary networks is not stable and usually requests a slower convergence speed with a small learning rate. Convolutions in classical CNNs are actually cross-correlation to measure the similarity of two inputs. Researchers and developers are used to taking convolution as a default operation to extract features from visual data, and introduce various methods to accelerate the convolution, even if there is a risk of sacrificing network capability. But there is hardly no attempt to replace convolution with another more efficient similarity measure that is better to only involve additions. In fact, additions are of much lower computational complexities than multiplications. Thus, we are motivated to investigate the feasibility of replacing multiplications by additions in convolutional neural networks.

\begin{figure*}[t]
	\centering
	\begin{tabular}{cc}
		\includegraphics[width=0.48\linewidth]{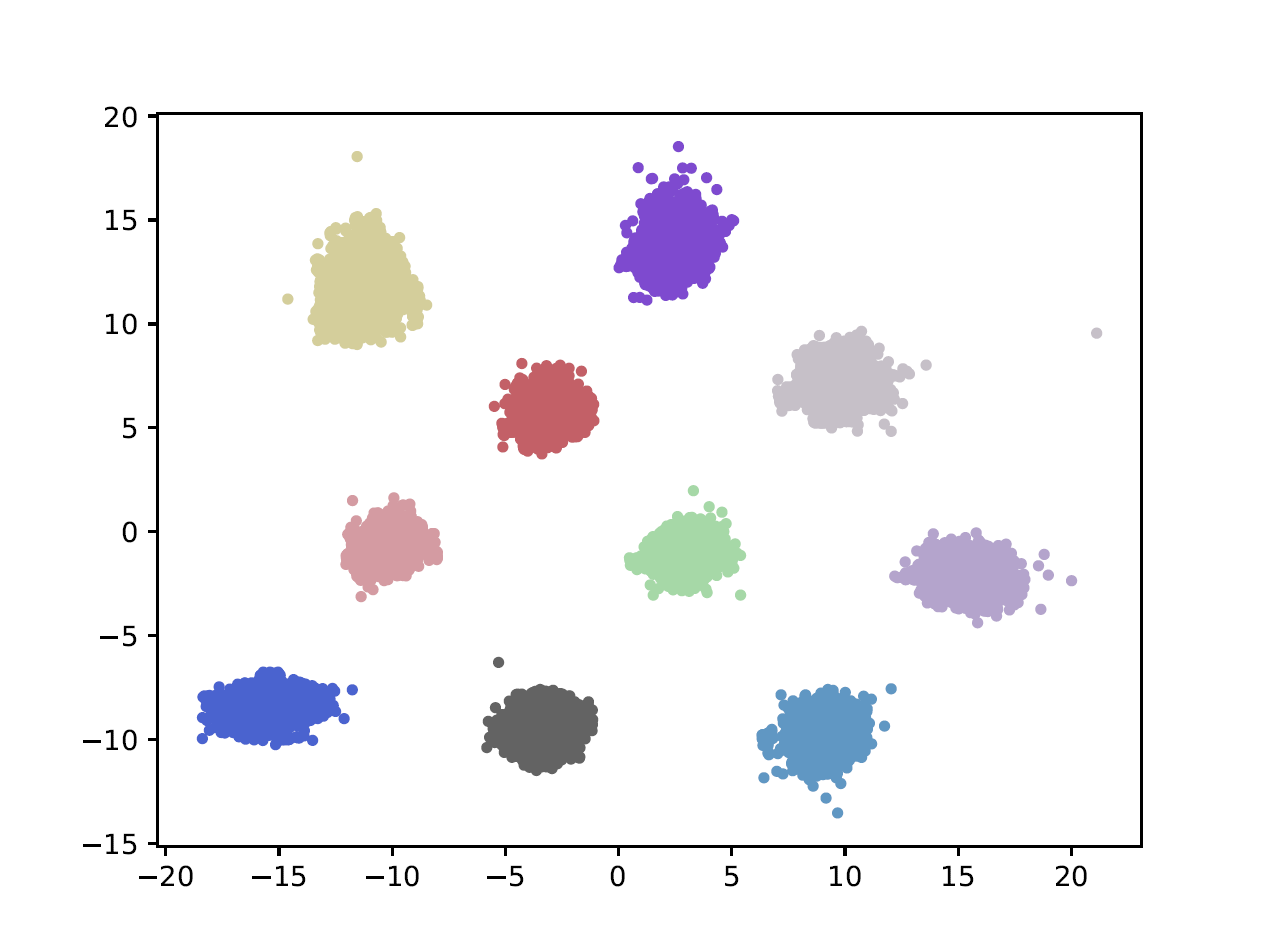} &
		\quad \includegraphics[width=0.48\linewidth]{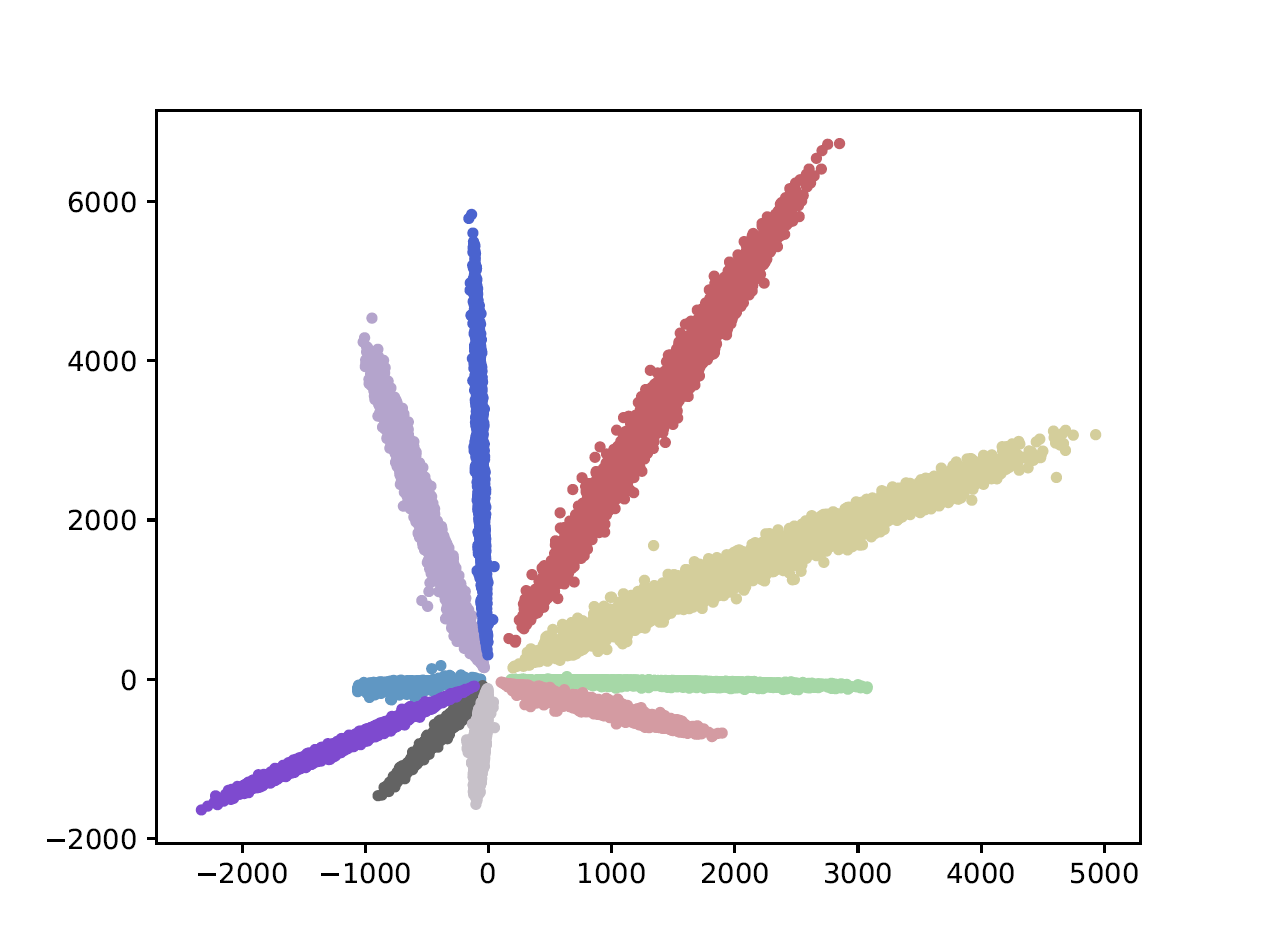} \\
		(a) Visualization of features in AdderNets  &(b)  Visualization of features in CNNs  \\
	\end{tabular}
	\caption{Visualization of features in AdderNets and CNNs. Features of CNNs in different classes are divided by their angles. In contrast, features of AdderNets tend to be clustered towards different class centers, since AdderNets use the $\ell_1$-norm to distinguish different classes. The visualization results suggest that $\ell_1$-distance can serve as a similarity measure the distance between the filter and the input feature in deep neural networks}
	\label{Fig:visualfea}
\end{figure*}

In this paper, we propose adder networks that maximize the use of addition while abandoning convolution operations. Given a series of small template as ``filters’’ in the neural network, $\ell_1$-distance could be an efficient measure to summarize absolute differences between the input signal and the template as shown in Figure~\ref{Fig:visualfea}. Since subtraction can be easily implemented through addition by using its complement code, $\ell_1$-distance could be a hardware-friendly measure that only has additions, and naturally becomes an efficient alternative of the convolution to construct neural networks. To give a theoretical guarantee for AdderNets, we prove that both the single hidden layer AdderNet and the width-bounded AdderNet can approximate any Lebesgue integrable function in a compact set. This result is comparable to the universal approximation results for traditional neural networks. We also present a approximation bound for AdderNets with a single hidden layer. To optimize the AdderNet effectively, we design the $\ell_2$ to $\ell_1$ training scheme and the adaptive learning rate scaling to ensure sufficient updates of the templates and a better network convergence. The proposed AdderNets are deployed on several benchmarks, and experimental results demonstrate that AdderNets can achieve comparable recognition accuracy to conventional CNNs.

A preliminary version of this work has been presented earlier~\cite{chen2020addernet}. The present work improves the initial version significantly. First, we propose the universal approximation theorem of two-layer AdderNet and width-bounded deep AdderNet, which lays the foundation for AdderNets in various deep learning-based applications. 
Second, we introduce a general $l_p$-AdderNet, where $1\leq p\leq2$. By gradually reducing $p$ from 2 to 1 during the iterations, the training difficulty of $l_1$-AdderNet could be reduced. 
In addition, extended experiments are conducted to show the approximation capacity of AdderNets and the improvement of the new training scheme. 

This paper is organized as follows. Section 2 revisits the related works on neural network compression. Section 3 proposes the AddderNets and discusses the universal approximation theorem of AdderNets. Section 4 introduces the optimization scheme of AdderNets using $\ell_2$ to $\ell_1$ training and adaptive learning rate scaling. Section 5 conducts experiments to show the approximation capacity and effectiveness of the proposed AdderNets on toy examples and classification tasks, respectively. Section 6 concludes this paper.

\section{Related works}\label{sec:related}

To reduce the computational complexity of convolutional neural networks, a number of works have been proposed for eliminating useless calculations. 

Pruning based methods aims to remove redundant weights to compress and accelerate the original network. Denton~\etal~\cite{SVD} decomposed weight matrices of fully-connected layers into simple calculations by exploiting singular value decomposition (SVD). Han~\etal~\cite{han2015deep} proposed discarding subtle weights in pre-trained deep networks to omit their original calculations without affecting the performance. Wang~\etal~\cite{wang2016cnnpack} further converted convolution filters into the DCT frequency domain and eliminated more floating number multiplications. In addition, Hu~\etal~\cite{Trimming} discarded redundant filters with less impacts to directly reduce the computations brought by these filters. Luo~\etal~\cite{luo2017thinet} discarded redundant filters according to the reconstruction error.  Hu~\etal~\cite{hu2020triple} proposed dubbed Robust Dynamic Inference Networks (RDI-Nets), which allows for each input to adaptively choose one of the multiple output layers to output its prediction. Wang~\etal~\cite{wang2019e2} proposed a E2-Training method, which can train deep neural networks with over 80\% energy savings.

Instead of directly reducing the computational complexity of a pre-trained heavy neural network, lot of works focused on designing novel blocks or operations to replace the conventional convolution filters. Howard~\etal~\cite{howard2017mobilenets} designed MobileNet, which decompose the conventional convolution filters into the point-wise and depth-wise convolution filters with much fewer FLOPs. Zhang~\etal~\cite{zhang2018shufflenet} combined group convolutions~\cite{ResNeXt} and a channel shuffle operation to build efficient neural networks with fewer computations. Wu~\etal~\cite{wu2018shift} presented a parameter-free ``shift" operation with zero flop and zero parameter to replace conventional filters and largely reduce the computational and storage cost of CNNs. Wang~\etal~\cite{Versatile} developed versatile convolution filters to generate more useful features utilizing fewer calculations and parameters. Xu~\etal~\cite{juefei2018perturbative} proposed perturbative neural networks to replace convolution and instead computes its response as a weighted linear combination of non-linearly activated additive noise perturbed inputs. Han~\etal~\cite{han2019ghostnet} proposed GhostNet to generate more features from cheap operations and achieve the state-of-the-art performance on lightweight architectures.

Besides eliminating redundant weights or filters in deep convolutional neural networks, Hinton~\etal~\cite{hinton2015distilling} proposed the knowledge distillation (KD)
scheme, which transfer useful information from a heavy teacher network to a portable student network by minimizing the Kullback-Leibler divergence between their outputs. Besides mimic the final outputs of the teacher networks, Romero~\etal~\cite{romero2014fitnets} exploit the hint layer to distill the information in features of the teacher network to the student network. You~\etal~\cite{you2017learning} utilized multiple teachers to guide the training of the student network and achieve better performance. Yim~\etal~\cite{yim2017gift} regarded the relationship between features from two layers in the teacher network as a novel knowledge and introduced the FSP (Flow of Solution Procedure) matrix to transfer this kind of information to the student network.

Nevertheless, the compressed networks using these algorithms still contain massive multiplications, which costs enormous computation resources. As a result, subtractions or additions are of much lower computational complexities when compared with multiplications. However, they have not been widely investigated in deep neural networks, especially in the widely used convolutional networks. Therefore, we propose to minimize the numbers of multiplications in deep neural networks by replacing them with subtractions or additions.

\section{Adder Networks}\label{sec:method}

Consider a filter $F\in \mathbb{R}^{d\times d\times c_{in}\times c_{out}}$ in an intermediate layer of the deep neural network, where kernel size is $d$, input channel is $c_{in}$ and output channel is $c_{out}$. The input feature is defined as $X\in \mathbb{R}^{H\times W \times c_{in}}$, where $H$ and $W$ are the height and width of the feature, respectively. The output feature $Y$ indicates the similarity between the filter and the input feature,  
\begin{equation}
\small
Y(m,n,t) = \sum_{i=0}^{d}\sum_{j=0}^{d}\sum_{k=0}^{c_{in}} S(X(m+i,n+j,k), F(i,j,k,t)), \label{fcn:conv}
\end{equation} 
where $S(\cdot, \cdot)$ is a pre-defined similarity measure.  If cross-correlation is taken as the metric of distance, \ie $S(x, y) = x \times y$, Eq. (\ref{fcn:conv}) becomes the convolution operation. Eq. (\ref{fcn:conv}) can also implies the calculation of a fully-connected layer when $d=1$. In fact, there are many other metrics to measure the distance between the filter and the input feature. However, most of these metrics involve multiplications, which bring in more computational cost than additions.

\subsection{Adder Networks}

We are therefore interested in deploying distance metrics that maximize the use of additions.  $\ell_1$ distance calculates the sum of the absolute differences of two points’ vector representations, which contains no multiplication. Hence, by calculating $\ell_1$ distance between the filter and the input feature, Eq. (\ref{fcn:conv}) can be reformulated as:
\begin{equation}
\small
Y(m,n,t) = -\sum_{i=0}^{d}\sum_{j=0}^{d}\sum_{k=0}^{c_{in}} \vert X(m+i,n+j,k) - F(i,j,k,t)\vert. \label{fcn:l1}
\end{equation}
Addition is the major operation in $\ell_1$ distance measure, since subtraction can be easily reduced to addition by using complement code. With the help of $\ell_1$ distance, similarity between the filters and features can be efficiently computed. 

Although both $\ell_1$ distance Eq.  (\ref{fcn:l1}) and cross-correlation in Eq. (\ref{fcn:conv}) can measure the similarity between filters and inputs, there are some differences in their outputs. The output of a convolution filter, as a weighted summation of values in the input feature map, can be positive or negative, but the output of an adder filter is always negative. Hence, we resort to batch normalization for help, and the output of adder layers will be normalized to an appropriate range and all the activation functions used in conventional CNNs can then be used in the proposed AdderNets. Although the batch normalization layer involves multiplications, its computational cost is significantly lower than that of the convolutional layers and can be omitted. Considering a convolutional layer with a filter $F\in \mathbb{R}^{d\times d\times c_{in}\times c_{out}}$, an input $X\in \mathbb{R}^{H\times W \times c_{in}}$ and an output $Y\in \mathbb{R}^{H'\times W' \times c_{out}}$, the computation complexity of convolution and batch normalization is $\mathcal{O}(d^2c_{in}c_{out}HW)$ and $\mathcal{O}(c_{out}H'W')$, respectively. In practice, given an input channel number $c_{in}=512$ and a kernel size $d=3$ in ResNet~\cite{he2016deep}, we have $\frac{d^2c_{in}c_{out}HW}{c_{out}H'W'}\approx 4068$. Since batch normalization layer has been widely used in the state-of-the-art convolutional neural networks, we can simply upgrade these networks into AddNets by replacing their convolutional layers into adder layers to speed up the inference and reduces the energy cost.

Intuitively, Eq. (\ref{fcn:conv}) has a connection with template matching~\cite{brunelli2009template} in computer vision, which aims to find the parts of an image that match the template. $F$ in Eq. (\ref{fcn:conv}) actually works as a template, and we calculate its matching scores with different regions of the input feature $X$. Since various metrics can be utilized in template matching, it is natural that $\ell_1$ distance can be utilized to replace the cross-correlation in Eq. (\ref{fcn:conv}). Note that there are few works also focus on utilizing different metrics in deep networks. Wang~\etal~\cite{wang2019kervolutional} aims to achieve high performance by employing complex metrics. Zhang~\etal~\cite{zhang2021towards} proposed to use $\ell_{\inf}$ distance as basic operation and directly provides a theoretical guarantee of the certified robustness. In contrast, we focus on the $\ell_1$ distance to minimize the energy consumption.

\subsection{Universal Approximation by AdderNets}

There have been many studies on the approximation capacity of neural networks. Hornik~\etal~\cite{hornik1989multilayer} first proved that feedforward networks with one hidden layer using an arbitrary squashing function as the activation function are capable of approximating any Borel measurable function from a finite dimensional space to another up to any desired accuracy. This property is referred to as universal approximation. 
Leshno~\etal~\cite{leshno1993multilayer} further extended the results to feedforward networks with non-polynomial activation functions. Besides the traditional feedforward networks, Schafer~\etal~\cite{schafer2006recurrent} proved that the universal approximation property of RNNs. Yun~\etal~\cite{yun2019transformers} demonstrated that transformers are universal approximators of sequence-to-sequence functions. 

Nevertheless, the universal approximation property of AdderNets has not been established. In fact, the AdderNet calculates the Manhattan distance between the input features and filters, which is fundamentally different to traditional feedforward networks using cross correlation. Thus, most operations in AdderNets are cheaper additions which are significantly more energy efficient than massive multiplications employed in the traditional neural networks. If we can utilize adder units to construct universal approximators, then the performance of using AdderNets on any deep learning based applications can be guaranteed.

Here we briefly review the annotation of each adder layer along with the batch normalization layer and  present two universal approximation theorems for AdderNets,  one for shallow networks with a single hidden layer and the other for deep networks with bounded width. 

\textbf{Definition 1.} For the arbitrary input data $\mathbf{X}\in \mathbb{R}^{d_{in}}$ the weight parameters $\mathbf{W} \in \mathbb{R}^{d_{in}\times d_{out}}$ where $d_{in},d_{out} \in \mathbb{N} $, a one-layer adder net $L$ using batch normalization for calculate the outputs $Y\in \mathbb{R}^{d_{out}}$ can be formulated as:
\begin{equation}
{L: \mathbb{R}^{d_{in}} \rightarrow \mathbb{R}^{d_{out}}: L(\mathbf{X})_i = a_i  \Vert \mathbf{W}_{i} - \mathbf{X}\Vert_1 +b_i},
\label{layer}
\end{equation}
where $\mathbf{W}_i \in \mathbb{R}^{d_{in}}$ is the $i$-th column in $\mathbf{W}$, $a_i,b_i \in \mathbb{R}$ are the scale factor and bias for the $i$-th output channel, respectively. Based on the above definition for a single adder layer, we then define the multi-layer adder network.

\textbf{Definition 2.} For an arbitrary activation function $G(\cdot)$ for mapping $\mathbb{R}$ to $\mathbb{R}$ and the $d$-dimentional input data $\mathbf{X}\in \mathbb{R}^{d}$ a $n$ layers AdderNet A for calculating a scalar can be formulated as 
\begin{equation}
{A: \mathbb{R}^{d} \rightarrow \mathbb{R}}: A(\mathbf{X}) = L_n(G(L_{n-1}(\dots G(L_1(\mathbf{X})) \dots))).
\end{equation}

\textbf{Definition 3.} A width-bounded AdderNet means the maximum number of neuron in each layer of this network is bounded by a given constant $w\in\mathbb{N}$.

Then, we aim to prove that both a two-layer AdderNet and a width-bounded AdderNet with the widely used non-linear activation function $G(x)= \mbox{ReLU}(x) = \max(x,0)$ is a universal approximator. Note that the above definition does not contain any convolutional layers, max pooling and others operations, because we investigate the AdderNet itself for constructing the universal approximator.



\begin{theorem}[Universal Approximation of Two-Layer AdderNet]
	\label{theorem1}
	For any $n \in \mathbb{N}$, the family of AdderNet with ReLU as activation function can universally approximate any $f\in \ell_1(K)$, where $K$ is a compact set of $\mathbb{R}^d$. In other words, for any $\epsilon > 0$, there is an two-layer AdderNet $A$ such that:
	\begin{equation}
	\int_K |f(\mathbf{X}) - A(\mathbf{X})| d\mathbf{X} \leq \epsilon ,
	\end{equation}
	where $\mathbf{X}\in K$.
\end{theorem}

In order to prove Theorem~\ref{theorem1}, we construct a set of function $g(\mathbf{X})=\sum_i a_i \mbox{ReLU}(\Vert \mathbf{W}_i - \mathbf{X} \Vert_1 + b_i)$, and show that it can be represented by a two-layer AdderNet. Thus, it is sufficient to prove that $g(\mathbf{X})$ is a universal approximator.

\begin{lemma}
	\label{lemma1}
	The set of functions $g(\mathbf{X})=\sum_i a_i \mbox{ReLU}(\Vert \mathbf{W}_i - \mathbf{X} \Vert_1 + b_i)$ can be represented by two-layer AdderNets, where $\mathbf{X}$ belongs to a compact set $K$ of $\mathbb{R}^{d}$ is the input vector, $\mathbf{W}_i\in \mathbb{R}^d, a_i\in \mathbb{R}, b_i \in \mathbb{R}$.
\end{lemma}
\begin{proof}
	We set the first layer of AdderNet as:
	
	\begin{equation}
	{L^1: \mathbb{R}^d \rightarrow \mathbb{R}^{t}: L^1(\mathbf{X})_i = \vert a_i\vert  \Vert \mathbf{W}_{i} - \mathbf{X}\Vert_1 +b_i},
	\end{equation}
	
	since $K$ is compact and $L^1$ is continuous, there exists $M \in \mathbb{R}$ such that: $\vert L^1(\mathbf{X})_i\vert < M , i = 1,2,\dots, t$. We then set the second layer as:
	
	\begin{equation}
	\begin{aligned}
	L^2: \mathbb{R}^{t} \rightarrow \mathbb{R}: L^2(x)& = \sum_{a_i>0} \vert \mathbf{X}_i + M \vert \\&+ \sum_{a_i \leq 0 } \vert \mathbf{X}_i - M\vert + (p_1+p_2)M,
	\end{aligned}
	\end{equation}
	where $p_1,p_2$ is the number of $i$ such that $a_i \geq 0$ or $a_i \leq 0$, respectively.
	
	Therefore, the two-layer AdderNet can be denote as:
	\begin{equation}
	\small
	\begin{aligned}
	A(x) =& \sum_{a_i>0} \vert \mbox{ReLU} (\vert a_i\vert  \Vert \mathbf{W}_{i} - \mathbf{X}\Vert_1 +b_i + M) \vert \\
	&+  \sum_{a_i \leq 0} \vert \mbox{ReLU} (\vert a_i\vert  \Vert\mathbf{W}_{i} - \mathbf{X}\Vert_1 +b_i) -M \vert - (p_1+p_2)M \\
	=&  \sum_{a_i>0} \mbox{ReLU} (\vert a_i\vert  \Vert \mathbf{W}_{i} - \mathbf{X}\Vert_1 +b_i) \\&-\sum_{a_i\leq 0} \mbox{ReLU} (\vert a_i\vert  \Vert \mathbf{W}_{i} - \mathbf{X}\Vert_1 +b_i)\\
	=& \sum_i a_i \mbox{ReLU}(\Vert \mathbf{W}_{i} - \mathbf{X}\Vert_1 + b_i) .
	\end{aligned}
	\end{equation}
\end{proof}

\begin{theorem}[\cite{park1991universal}]
	\label{theorem:rbf}
	The radio basis function networks:
	\begin{equation}
	\sum_{i}  a_i K(\frac{\mathbf{X}- \mathbf{Z}_i }{\sigma})
	\end{equation} is dense in $\ell_1(\mathbb{R}^d)$, where $\mathbf{X},\mathbf{Z}_i\in\mathbb{R}^d$, if $K$ is integrable bounded function such that $K$ is continuous almost everywhere and $\int K(\mathbf{X}) d\mathbf{X} \neq 0$.
\end{theorem}

Theorem~\ref{theorem:rbf} is the universal approximation theorem for the RBF networks, which is proved by Park and Sandberg~\cite{park1991universal}.
We are now ready to present the proof of Theorem~\ref{theorem1}. 
\begin{proof} (of Theorem~\ref{theorem1})
	From Lemma~\ref{lemma1}, two-layer AdderNets can present the set of functions $g(\mathbf{X})=\sum_i a_i \mbox{ReLU}(\Vert \mathbf{W}_i - \mathbf{X} \Vert_1 + b_i)$, while $g(\mathbf{X})$ is dense in $\ell_1(K)$ according to
	the universal approximation theorem of radio basis function (RBF) networks (theorem~\ref{theorem:rbf}). 
	We conclude that the class of two-layer AdderNets is dense in $\ell_1(K)$, where $K$ is an arbitrary compact set in $\mathbb{R}^d$. 
\end{proof}

Though we show that AdderNets are universal approximators, the expressive power of AdderNets is still unknown, \ie how many hidden units should we use to achieve a required approximation error. Therefore, we present a approximation bound for AdderNets with a single hidden layer.

\begin{theorem}
	Assume that $f$ is Lipschitz: $\exists L_1>0$ such that $| f(\mathbf{X})-f(\mathbf{X}')| \leq L_1 \|\mathbf{X}-\mathbf{X}'\|_1$ for all $\mathbf{X}$ and $\mathbf{X}'$, and assume $\Vert f\Vert_1 = \int \vert f(\mathbf{X}) \vert d\mathbf{X} < \infty$.
	Given any probability measure $\mu$, and define $\|f\|_\mu= \int |f(\mathbf{X})| d \mu(\mathbf{X}) $. There exists a single hidden layer AdderNet $A$ with ReLU activation function of width no more than 
	\[
	\frac{\|f\|_1 \|f\|_\mu}{\min(\epsilon^{d+2} L_1^2 , \epsilon^{d+1} L_1 \|f\|_\mu)}
	\]
	such that
	\[
	\int (f(\mathbf{X})-A(\mathbf{X}))^2 d \mu(\mathbf{X})
	= O(\epsilon^2 L_1^2) ,
	\]
	where $O(\cdot)$ contains a $d$-dependent constant. 
\end{theorem}
\begin{proof}
	By lemma~\ref{lemma1}, we can use $g(\mathbf{X})$ to approximate $f$, where $g(\mathbf{X})$ contains:
	\begin{equation}
	\phi_N(\mathbf{X}) = \frac{1}{N} \sum_{i=1}^n a_i r_{\epsilon}(\Vert \mathbf{X} - \mathbf{W}_i \Vert),\label{eq:phi-N}
	\end{equation}
	where $r_{\epsilon}(x) = \frac{r(\frac{x}{\epsilon})}{\epsilon^{d}}$ and $r(x)= \max(0,x+1) + \max(0,x-1)-2\max(0,x)$.
	
	We want to use $\phi_N(\mathbf{X})$ to approximate $f$. To do so, we first approximate $f$ by continuous approximation below:
	\begin{equation}
	\psi_\epsilon(\mathbf{X}) = \int f(\mathbf{Z})r_\epsilon(\Vert \mathbf{X} - \mathbf{Z}\Vert_1) c_0^{-1}d\mathbf{X},
	\end{equation}
	where $c_0 =\int g(\Vert \mathbf{X}\Vert_1)d\mathbf{X}$.
	
	Since $f$ is Lipschitz, we have
	\begin{equation}
	\begin{aligned}
	|f(\mathbf{X}) -\psi_\epsilon(\mathbf{X})| &\leq \int |f(\mathbf{X})-f(\mathbf{X}-\epsilon\mathbf{Z})|r(\Vert \mathbf{Z}\Vert_1) c_0^{-1}d\mathbf{Z}\\
	&\leq \epsilon L_1 c_0^{-1}c_1,
	\end{aligned}
	\end{equation}
	where $c_1 = \int \Vert\mathbf{X}\Vert_1 r(\Vert\mathbf{X}\Vert_1) d\mathbf{X}$

	Let $q(x) =\frac{ \vert f(\mathbf{X})\vert}{\Vert f \Vert_1} $. We draw $\mathbf{Z}_1,\dots,\mathbf{Z}_N \sim q(\mathbf{Z})$, and set $a_i = \mbox{sgn} (f(\mathbf{Z}_i)) \Vert f\Vert_1 c_0^{-1}$ in \eqref{eq:phi-N}.  This implies
	\begin{equation}
	E_{\mathbf{Z}_1,\dots,\mathbf{Z}_n}
	\phi_N(\mathbf{X})
	=\psi_\epsilon(\mathbf{X}).
	\label{eq:phi-N-exp}
	\end{equation}
	It follows that
	\begin{equation}
	\begin{aligned}
	& E_{\mathbf{Z}_1,\dots,\mathbf{Z}_n} (\psi_\epsilon(\mathbf{X})-\phi_N(\mathbf{X}))^2 \\\leq&  \frac{\Vert f\Vert_1^2}{N} E_{\mathbf{Z}_1\sim q(\mathbf{Z})} c_0^{-2} r_\epsilon(\Vert \mathbf{X}-\mathbf{Z}_1\Vert_1)^2\\
	=&  \frac{\Vert f\Vert_1}{N} \int |f(\mathbf{Z})| c_0^{-2} r_\epsilon(\Vert \mathbf{X}-\mathbf{Z}\Vert_1)^2 d\mathbf{Z}\\
	= &\frac{\Vert f\Vert_1}{N\epsilon^dc_0^2} \int |f(\mathbf{X}-\epsilon\mathbf{Z})| c_0^{-2} r(\Vert \mathbf{Z}\Vert_1)^2 d\mathbf{Z}  \\
	\leq &\frac{\Vert f\Vert_1}{N\epsilon^dc_0^2} \int(|f(\mathbf{X})| +G \epsilon\|\mathbf{Z}\|_1) c_0^{-2} r(\Vert \mathbf{Z}\Vert_1)^2 d\mathbf{Z}\\ =&  O (N^{-1} \epsilon^{-d}\|f\|_1 [|f(\mathbf{X})|+ L_1\epsilon] ) ,
	\end{aligned}
	\end{equation}
	where $O(\cdot)$ a $d$-dependent constant.
	From \eqref{eq:phi-N-exp}, we have:
	\begin{equation}
	\begin{aligned}
	&E (\phi_N(\mathbf{X})-f(\mathbf{X}))^2\\
	=&
	E(\phi_N(\mathbf{X})-\psi_\epsilon(\mathbf{X}))^2 + E(\psi_\epsilon(\mathbf{X})-f(\mathbf{X}))^2 \\
	=& 
	O (N^{-1} \epsilon^{-d}\|f\|_1 [|f(\mathbf{X})|+ L_1\epsilon] + \epsilon^2 L_1^2),
	\end{aligned}
	\end{equation}
	where $O(\cdot)$ contains a $d$-dependent constant.
	Taking expectation over $\mu$, we obtain
	\begin{equation}
	\begin{aligned}
	&	E \int (\phi_N(\mathbf{X})-f(\mathbf{X}))^2 d \mu(\mathbf{X})\\
	=& O (N^{-1} \epsilon^{-d}\|f\|_1 [\|f(\mathbf{X})\|_\mu+ L_1\epsilon] + \epsilon^2 L_1^2) .
	\end{aligned}
	\end{equation}
	We may pick optimal $N$ with $\|f\|_1 N^{-1} \epsilon^{-d}= c \min(\epsilon^{2}\|f\|_\mu^{-1} L_1^2 , \epsilon L_1)$, for a constant $c$. Note that with this choice, the desired bound holds for the expectation over the choices of $z_1,\ldots,z_N$ of \eqref{eq:phi-N}. Therefore there exists a choice of $z_1,\ldots,z_N$ such that the desired bound holds, which implies the theorem.
\end{proof}

This theorem shows that two-layer AdderNet achieve the error less than $\epsilon$ by using $O(\epsilon^{-1})$ hidden units (choosing $d=0$ or $d=-1$). In other words, with $n$ hidden units, AdderNets achieve a integrated squared error of order $O(1/n)$. This result is same as the bound in traditional networks~\cite{barron1993universal}, indicating that the two different networks have similar expressive power.

Next, we begin to the universal approximation theorem for deep AdderNets with a bounded width, which is similar with the important results in traditional networks~\cite{lu2017expressive}. We first present a lemma to represent a set of fully-connect layers by two-layer AdderNets.

\begin{lemma}
	\label{lemma2}
	A fully-connected layer $C_1$ with $m$ outputs whose weights are same or zero along each dimension:
	\begin{equation}
	C_1: \mathbb{R}^d \rightarrow \mathbb{R}^m : C(\mathbf{X})_i = \mathbf{A}_{i} \sum_{j=1}^d \mathbf{B}_{ij}  \mathbf{X}_j,
	\end{equation}
	where $A\in \mathbb{R}^m$, $\mathbf{B}_{ij}$ is 0 or 1, and $\mathbf{X} \in \mathbb{R}^d$, can be represented by two-layer AdderNets with $2m+2$ hidden units.
\end{lemma}

\begin{proof}
	
	We first show that an adder layer can be constructed to a fully-connected layers $C_2$ with only $+a$ or $-a$ weights in each dimension whose weights are same along each dimension:
	\begin{equation}
	\small
	\begin{aligned}
	L_A(\mathbf{X})_i & =  a_i  \sum_{j} \vert \mathbf{W}_{ij} - \mathbf{X}_j  \vert - \Vert \mathbf W_{i}\Vert_1 + b_i\\
	& = a_i  \sum_{j} \mbox{sgn} (\mathbf{W}_{ij}) (\mathbf{W}_{ij} - \mathbf{X}_j ) - \sum_{j} \mbox{sgn} (\mathbf{W}_{ij})  \mathbf W_{ij} + b_i\\
	& = a_i  \sum_{j} \mbox{sgn} (-\mathbf{W}_{ij}) \mathbf{X}_j + b_i,
	\end{aligned}
	\end{equation}
	where we set $\mathbf{W}_{ij} > \max_k \vert\mathbf{X}_k\vert$ (since the function is defined in a compact set, $\mathbf{X}$ in each layer is bounded).
	
	Now we begin to construct two-layer AdderNet to $C_1$. For the fist layer $L_1$, we set the $i$-th in $L_1$ as: $\mbox{sgn} (-\mathbf{W}^1_{ij})$ is 1 or -1 if $\mathbf{B}_{ij}$ is $1$ or 0. The $2i$-th neuron as: $\mbox{sgn} (-\mathbf{W}^1_{(2i)j})$ is -1 or 1 if $\mathbf{B}_{ij}$ is $1$ or 0. Then the weight of the second last and last neuron ($2m+1$ and $2m+2$) are all $1$ and $-1$, respectively. For the second layer, the $i$-th neuron is set as: $\mbox{sgn} (-\mathbf{W}^2_{ij})=-1$ for the $j=2i$ and $j=2m+2$ and 1 for others. Thus, the output of $j$-th neuron in layer two is:
	\begin{equation}
	\small
	\begin{aligned}
	L_2(L_1(\mathbf{X}))_i =& a_i \sum_j \mbox{sgn} (-\mathbf{W}^2_{ij}) L_1(\mathbf{X})_j \\
	=& a_i ( L_1(\mathbf{X})_i -L_1(\mathbf{X})_{2i} + L_1(\mathbf{X})_{2m+1} \\
	&-L_1(\mathbf{X})_{2m+2} )\\
	=& 2a_i (\sum_j \mathbf{1}_{\mathbf{B}_{ij}>0} \mathbf{X}_j -  \sum_j \mathbf{1}_{\mathbf{B}_{ij}=0} \mathbf{X}_j + \sum_j  \mathbf{X}_j )\\
	=& 4a_i \sum_{j=1} \mathbf{B}_{ij}  \mathbf{X}_j.
	\end{aligned}    
	\end{equation}
	Note that the ReLU activation between two layer can be easily ignored by selecting large enough bias for the first layer and then the bias can be cut back by the weights in the second layer.
\end{proof}

We are now ready to proof universal approximation of width-bounded AdderNet.
\begin{theorem}[Universal Approximation of Width-Bounded AdderNet]
	\label{theorem:width}
	For any $d \in \mathbb{N}$, the family of width-bounded AdderNet with ReLU as activation function can universally approximate any $f\in \ell_1(K)$, where $K$ is a compact set of $\mathbb{R}^d$. More precisely, for any $\epsilon > 0$, there is a deep AdderNet $A$ with maximum width $w \leq2(d + 5)$ such that:
	\begin{equation}
	\int_K |f(\mathbf{X}) - A(\mathbf{X})| d\mathbf{X} \leq \epsilon ,
	\end{equation}
	where $\mathbf{X}\in K$.
\end{theorem}
\begin{proof}
	According to Lu~\etal~\cite{lu2017expressive}, any Lebesgue-integrable function can be approximated by a fully-connected ReLU network with width $w \leq d + 4$, which is proved by construction. We can find that the construction use fully-connected layers whose weights are same or zero along each dimension. According to lemma~\ref{lemma2}, we can use two adder layers to replace each fully-connected layers this construction, and the width for adder layers is less than  $2(d+5)$. 
\end{proof}

This theorem can be regarded as a dual version of Theorem 1, which proves that the two-layer AdderNets (\ie depth-bounded AdderNets) are universal approximator. In fact, the width-bounded traditional neural networks have also been proved to be universal approximators~\cite{lu2017expressive} and the bound of width is $d+4$. Our bound for AdderNets ($2(d+5)$) has the same magnitude with that of traditional networks, \ie $O(d)$, which implies that the AdderNets has similar approximation capacity with traditional networks.

\section{Optimization}

In the above section, we present that AdderNets are universal approximator, which provide a strong theoretical guarantee. We then turn to analyze how to train AdderNets to achieve good performance.

\subsection{Gradient in AdderNets}
Neural networks utilize back-propagation to compute the gradients of filters and stochastic gradient descent to update the parameters. In CNNs, the partial derivative of output features $Y$ with respect to the filters $F$ is calculated as:
\begin{equation}
\frac{\partial Y(m,n,t)}{\partial F(i,j,k,t)} = X(m+i,n+j,k), 
\end{equation} 
where $i\in [m,m+d]$ and $j \in [n,n+d]$. To achieve a better update of the parameters, it is necessary to derive informative gradients for SGD. In AdderNets, the partial derivative of $Y$ with respect to the filters $F$ is:
\begin{equation}
\frac{\partial Y(m,n,t)}{\partial F(i,j,k,t)} = \mbox{sgn} (X(m+i,n+j,k) - F(i,j,k,t)), \label{fcn:l1bp}
\end{equation}  
where $\mbox{sgn}(\cdot)$ denotes the sign function and the value of the gradient can only take +1, 0, or -1.

Considering the derivative of $\ell_2$-norm 
\begin{equation}
\frac{\partial Y(m,n,t)}{\partial F(i,j,k,t)} =  X(m+i,n+j,k) - F(i,j,k,t), \label{fcn:l2bp}
\end{equation} 
Eq. (\ref{fcn:l1bp}) can therefore lead to a signSGD~\cite{bernstein2018signsgd} update of $\ell_2$-norm. However, signSGD almost never takes the direction of steepest descent and the direction only gets worse as dimensionality grows~\cite{bernstein2018convergence}. It is unsuitable to optimize the neural networks of a huge number of parameters using signSGD. Therefore, we propose using Eq. (\ref{fcn:l2bp}) to update the gradients in our AdderNets.

\subsection{$\ell_2$ to $\ell_1$ Training}

It is common in deep learning to lean a ``weak'' model with the help of a ``strong'' model. For example, to train a low bit-width nerual network~\cite{zhang2021towards,yang2019quantization}. Motivated by the full-precision gradient, we further introduce the $\ell_2$-AdderNets, which calculate $\ell_2$ distance between the filter and the input feature, the filters in $\ell_2$-AdderNets can be reformulated as 
\begin{equation}
\small
Y(m,n,t) = -\sum_{i=0}^{d}\sum_{j=0}^{d}\sum_{k=0}^{c_{in}} \big[ X(m+i,n+j,k) - F(i,j,k,t)\big]^2.
\label{fcn:l2}
\end{equation} 

In fact, the output of the  $\ell_2$-AdderNets can be calculated as 
\begin{equation}
\small
\begin{aligned}
Y_{\ell_2}(m,n,t) =& -\sum_{i=0}^{d}\sum_{j=0}^{d}\sum_{k=0}^{c_{in}} \big[ X(m+i,n+j,k) - F(i,j,k,t)\big]^2\\
=&\sum_{i=0}^{d}\sum_{j=0}^{d}\sum_{k=0}^{c_{in}} \big[ 2  X(m+i,n+j,k)\times F(i,j,k,t)\\
&- X(m+i,n+j,k)^2 - F(i,j,k,t)^2\big]\\
=& 2Y_{CNN}(m,n,t)- \sum_{i=0}^{d}\sum_{j=0}^{d}\sum_{k=0}^{c_{in}} \big[ X(m+i,n+j,k)^2\\
& + F(i,j,k,t)^2\big].
\end{aligned}
\label{fcn:equ}
\end{equation} 
$\sum_{i=0}^{d}\sum_{j=0}^{d}\sum_{k=0}^{c_{in}}F(i,j,k,t)^2$ is same for each channel (\ie each fixed $t$). $\sum_{i=0}^{d}\sum_{j=0}^{d}\sum_{k=0}^{c_{in}} X(m+i,n+j,k)^2$ is the $\ell_2$-norm of each input patch. If this term is same for each patch, the output of $\ell_2$-AdderNet can be roughly seen as a linear transformation of the output of CNN. Although the $\ell_2$-AdderNet can achieve comparable performance with CNNs, its calculation contain square operations, which introduce multiplications and would bring large energy consumption compared with the $\ell_1$-AdderNet.

To this end, we propose an $\ell_2$ to $\ell_1$ training strategy to utilize the ability of $\ell_2$ norm to guide the training of $\ell_1$-AdderNet. We introduce the $\ell_p$-AdderNets ($1\leq p \leq 2$), which can be formulated as:
\begin{equation}
 \small
 Y(m,n,t) = -\sum_{i=0}^{d}\sum_{j=0}^{d}\sum_{k=0}^{c_{in}} \vert X(m+i,n+j,k) - F(i,j,k,t)\vert^p.
 \label{fcn:lp}
\end{equation}
Since it is difficult to directly training the $\ell_1$-AdderNet, we train the $\ell_2$-AdderNets at the beginning of training. During the training procedure, $p$ is linearly reduced from 2 to 1. Therefore, the $\ell_p$-AdderNet becomes $\ell_1$-AdderNet at the end of training. Since it is easy to find that the Proposition 2 is also right for $\ell_p$-AdderNet when $1<p<2$, the partial derivative of output features $Y$ in $\ell_p$-AdderNet with respect to the filters $F$ is calculated by Equ. (\ref{fcn:l2bp}). 

Note that we do not use the full precision gradient like Equ. (\ref{fcn:l2bp}) since the derivative of $X$ would influence the gradient in not only current layer but also layers before the current layer according to the gradient chain rule, and the change of gradients will make the training unstable. The partial derivative of output features $Y$ in $\ell_p$-AdderNet with respect to the input features $X$ is calculated as:
\begin{equation}
\frac{\partial Y(m,n,t)}{\partial X(i,j,k,t)} = \big[ F(i,j,k,t)-X(m+i,n+j,k)\big]^{p-1}.
\label{fcn:lpbp}
\end{equation}

\subsection{Adaptive Learning Rate Scaling}\label{sec:2.3}

In conventional CNNs, assuming that the weights and the input features are independent and identically distributed following normal distribution, the variance of the output can be roughly estimated as:
\begin{equation}
\begin{aligned}
Var[Y_{CNN}]& =  \sum_{i=0}^{d}\sum_{j=0}^{d}\sum_{k=0}^{c_{in}} Var[X\times F] \\
&= d^2c_{in} Var[X]Var[F].
\end{aligned}
\label{varcnn}
\end{equation} 
If variance of the weight is $Var[F]= \frac{1}{d^2c_{in}}$, the variance of output would be consistent with that of the input, which will be beneficial for the information flow in the neural network. In contrast, for AdderNets, the variance of the output can be approximated as:
\begin{equation}
\begin{aligned}
Var[Y_{AdderNet}] &= 
\sum_{i=0}^{d}\sum_{j=0}^{d}\sum_{k=0}^{c_{in}} Var[| X -F|] \\
&= \sqrt{\frac{\pi}{2}}d^2c_{in} (Var[X]+Var[F]),
\end{aligned}
\label{varadd}
\end{equation} 
when $F$ and $X$ follow normal distributions. In practice, the variance of weights $Var[F]$ is usually very small~\cite{glorot2010understanding}, \eg $10^{-3}$ or $10^{-4}$ in an ordinary CNN. Hence, compared with multiplying $ Var[X]$ with a small value in Eq. (\ref{varcnn}), the addition operation in Eq. (\ref{varadd}) tends to bring in a much larger variance of outputs in AdderNets. 

We next proceed to show the influence of this larger variance of outputs on the update of AdderNets. To promote the effectiveness of activation functions, we introduce batch normalization after each adder layer. Given input $x$ over a mini-batch $\mathcal{B} = \left\{x_{1}, \cdots, x_{m}\right\}$, the batch normalization layer can be denoted as:
\begin{equation}
y = \gamma \frac{x-\mu_{\mathcal{B}}}{\sigma_{\mathcal{B}}} + \beta,
\end{equation} 
where $\gamma$ and $\beta$ are parameters to be learned, and $\mu_{\mathcal{B}}= \frac{1}{m}\sum_i x_i$ and $\sigma^2_{\mathcal{B}}= \frac{1}{m}\sum_i (x_i-\mu_{\mathcal{B}})^2$ are the mean and variance over the mini-batch, respectively. The gradient of loss $\ell$ with respect to $x$ is then calculated as:
\begin{equation}
\small
\frac{\partial \ell}{\partial x_i} = \sum_{j=1}^{m} \frac{\gamma}{m^2 \sigma_{\mathcal{B}}} \left\{ \frac{\partial \ell}{\partial y_i} - \frac{\partial \ell}{\partial y_j}[ 1 +  \frac{(x_i-x_j)(x_j- \mu_{\mathcal{B}})}{\sigma_{\mathcal{B}} }] \right\}.
\label{fcn:bn}
\end{equation} 
Given a much larger variance $Var[Y] = \sigma_{\mathcal{B}}$ in  Eq. (\ref{varadd}), the magnitude of the gradient w.r.t $X$ in AdderNets would be much smaller than that in CNNs according to Eq. (\ref{fcn:bn}), and then the magnitude of the gradient w.r.t the filters in AdderNets would be decreased as a result of gradient chain rule.

\begin{algorithm}[t]
	\caption{The feed forward and back propagation of adder neural networks.} 
	\label{alg1} 
	\begin{algorithmic}[1] 
		\REQUIRE 
		An initialized $\ell_p$-adder network $\mathcal{N}$ and its training set $\mathcal{X}$ and the corresponding labels $\mathcal{Y}$, the global learning rate $\gamma$, $p=2$ and the hyper-parameter $\eta$. 
		\REPEAT
		\STATE Randomly select a batch $\{(\mbox{x},\mbox{y})\}$ from $\mathcal{X}$ and $\mathcal{Y}$;
		\STATE Employ the $\ell_p$-AdderNet $\mathcal{N}$ on the mini-batch: $\mbox{x} \rightarrow \mathcal{N}(\mbox{x})$;
		\STATE Calculate the derivative $\frac{\partial Y}{\partial F}$ and $\frac{\partial Y}{\partial X}$ for adder filters using Eq. (\ref{fcn:l2bp}) and Eq. (\ref{fcn:lpbp});
		\STATE Exploit the chain rule to generate the gradient of parameters in $\mathcal{N}$;
		\STATE Calculate the adaptive learning rate $\alpha_l$ for each adder layer according to Eq. (\ref{fcn:lr2}).
		\STATE Update the parameters in $\mathcal{N}$ using stochastic gradient descent.
		\STATE Decrease $p$ linearly if $p\geq1$.		
		\UNTIL convergence
		\ENSURE A well-trained $\ell_1$-adder network $\mathcal{N}$ with almost no multiplications.
	\end{algorithmic} 
\end{algorithm}

\begin{table}[h]
	\centering
	\caption{The $\ell_2$-norm of gradient of weight in each layer using different networks at 1st iteration.}
	\begin{tabular}{|c|c|c|c|}
		\hline
		\textbf{Model} & \textbf{Layer 1} & \textbf{Layer 2} & \textbf{Layer 3} \\
		\hline	
		\hline	
		AdderNet & 0.0009 & 0.0012 & 0.0146 \\
		\hline	
		CNN & 0.2261 & 0.2990 & 0.4646 \\
		\hline
	\end{tabular}
	\label{tab:weight}
\end{table}

Table~\ref{tab:weight} reports the $\ell_2$-norm of gradients of filters $\Vert F \Vert _2$ in LeNet-5-BN using CNNs and AdderNets on the MNIST dataset during the 1st iteration. LeNet-5-BN denotes the LeNet-5~\cite{lenet} adding an batch normalization layer after each convolutional layer. As shown in this table, the norms of gradients of filters in AdderNets are much smaller than that in CNNs, which could slow down the update of filters in AdderNets. 

A straightforward idea is to directly adopt a larger learning rate for filters in AdderNets. However, it is worth noticing that the norm of gradient differs much in different layers of AdderNets as shown in Table~\ref{tab:weight}, which requests special consideration of filters in different layers. To this end, we propose an adaptive learning rate for different layers in AdderNets. Specifically, the update for each adder layer $l$ is calculated by:
\begin{equation}
\Delta F_l = \gamma \times \alpha_l \times \Delta L(F_l),
\label{fcn:lr1}
\end{equation}
where $\gamma$ is a global learning rate of the whole neural network (\eg for adder and BN layers), $\Delta L(F_l)$ is the gradient of the filter in layer $l$ and $\alpha_l$ is its corresponding local learning rate. As filters in AdderNets act subtraction with the inputs, the magnitude of filters and inputs are better to be similar to extract meaningful information from inputs. Because of the batch normalization layer, the magnitudes of inputs in different layers have been normalized, which then suggests a normalization for the magnitudes of filters in different layers. The local learning rate can therefore be defined as:
\begin{equation}
\alpha_l = \frac{\eta\sqrt{k}}{\Vert \Delta L(F_l)\Vert_2},
\label{fcn:lr2}
\end{equation}
where $k$ denotes the number of elements in $F_l$, and  $\eta$ is a hyper-parameter to control the learning rate of adder filters. By using the proposed adaptive learning rate scaling, the adder filters in different layers can be updated with nearly the same step. The training procedure of the proposed AdderNet is summarized in Algorithm~\ref{alg1}.

\section{Experiment}\label{sec:experi}

In this section, we first conduct toy experiments to show the approximation capacity of AdderNets. We further implement experiments to validate the effectiveness of the proposed AdderNets on several benchmark datasets, including MNIST, CIFAR and ImageNet. Ablation study and visualization of features are provided to further investigate the proposed method. The experiments are conducted on NVIDIA Tesla V100 GPU in PyTorch.

\subsection{Toy Experiments of Approximation Capacity}

In the above subsections, we have proved that a two-layer adder neural network with a single hidden layer can be regarded as a universal approximator. Here we will further verify  the efficiency of AdderNet based universal approximation using some classical toy  classification datasets.

\textbf{Comparisions of AdderNets and traditional neural networks.} To achieve a fair comparison, we initialize a two-layer AdderNet and a two-layer feedforward neural networks with $n$ hidden units. These networks are optimized using SGD with Nesterov's Accelerated Gradient (NAG) . Weight decay and momentum are set as $5\times 10^{-4}$ and 0.9, respectively. Then, we train the two networks for 10,000 iterations using cosine learning rate schedule with an initial learning rate of 0.1. In addition, we use the binary cross entropy loss with sigmoid function for the binary classification task. For classification, the output is classified by whether it is larger than 0.5.

\textbf{Unit Ball.} The training set consists of random  samples $\{(x_i, y_i)\}$ generated from a two-dimensional normal distribution with mean of 0 and variance of 10. The label $z_i$ for the input sample $(x_i,y_i)$ is 
\begin{equation}
z_i=\left\{
\begin{aligned}
1 & , & \sqrt{x_i^2+y_i^2} < 10 , \\
0 & , & \sqrt{x_i^2+y_i^2} > 15.
\end{aligned}
\right.
\end{equation}
where a margin is created between positive and negative samples to make classification easier.

\begin{figure}[h]
	\centering
	\begin{tabular}{cccc}
		Training Data  &  $n=1$  &  $n=2$ &  $n=3$ \\
		\includegraphics[width=0.2\linewidth]{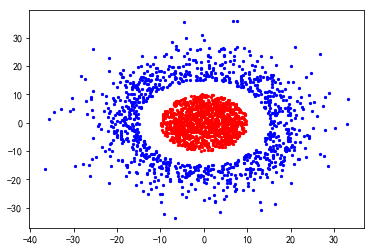}&  \includegraphics[width=0.2\linewidth]{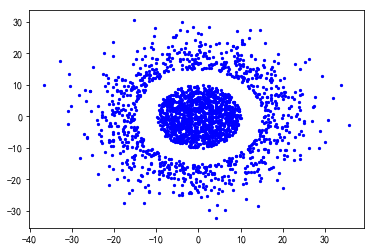} &  \includegraphics[width=0.2\linewidth]{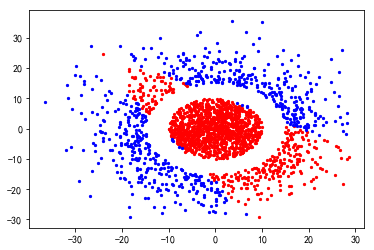}& \includegraphics[width=0.2\linewidth]{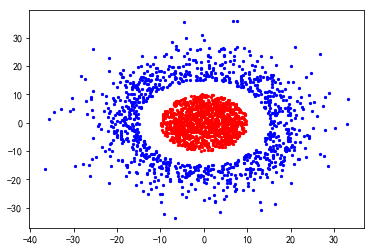}\\
		&   \includegraphics[width=0.2\linewidth]{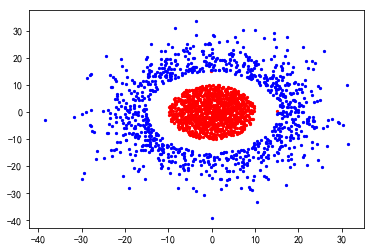} &  \includegraphics[width=0.2\linewidth]{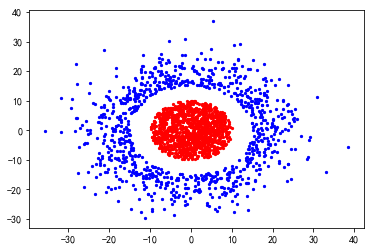}& \includegraphics[width=0.2\linewidth]{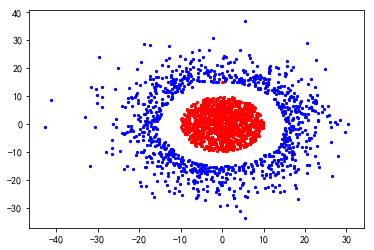}
	\end{tabular}
	\caption{Decision boundaries of classifying unit balls obtained by training fully connected networks (top row) and AdderNets (bottom row) with different number of hidden units. }
	\label{fig:ball}
\end{figure}

Figure~\ref{fig:ball} shows the decision boundaries of classifying unit balls obtained by training fully connected traditional neural networks and AdderNets with two layers, $n$ hidden units and ReLU activation functions. As the figure shows, the traditional network using multiplications fails to classify these points when $n=1$ and $n=2$ and obtains good results when $n=3$. In fact, the traditional network utilizes the cross-correlation between input data and weight parameters (\emph{i.e.}, the classifier) to calculate the output. Thus the classification results of the given data shown in figrue~\ref{fig:ball} are divided by lines. In contrast, AdderNet utilizes the $\ell_1$-distance and the points are divided by centrals. So, AdderNet can well distinguish all input samples even when $n=1$.

\textbf{Multiple Unit Balls.} We further construct the classification tasks on the multiple unit ball dataset, where the label $z_i$ for the input sample $(x_i,y_i) $ is calculated as
\begin{equation}
z_i=\left\{
\begin{aligned}
1 & , & \sqrt{(x_i-10)^2+(y_i-10)^2} < 10, \\
1 & , &\sqrt{(x_i+10)^2+(y_i+10)^2} < 10, \\
0 & , & \mbox{otherwise},
\end{aligned}
\right.
\end{equation}
where $\{(x_i, y_i)\}$ generated from a two-dimensional normal distribution with mean of 0 and variance of 15.

\begin{figure}[h]
	\centering
	\begin{tabular}{cccc}
		Training Data  &  $n=3$  &  $n=5$ &  $n=10$ \\
		\includegraphics[width=0.2\linewidth]{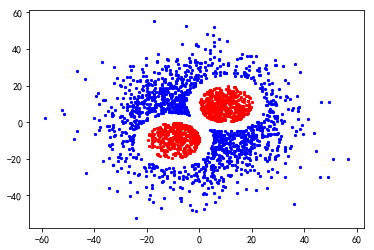}&  \includegraphics[width=0.2\linewidth]{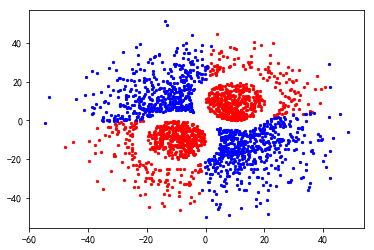} &  \includegraphics[width=0.2\linewidth]{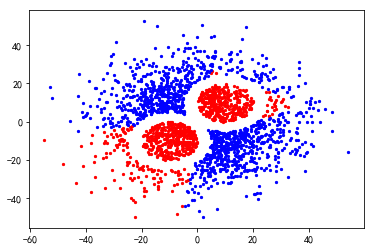}& \includegraphics[width=0.2\linewidth]{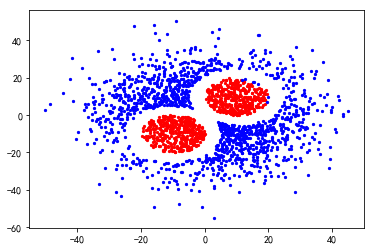}\\
		&   \includegraphics[width=0.2\linewidth]{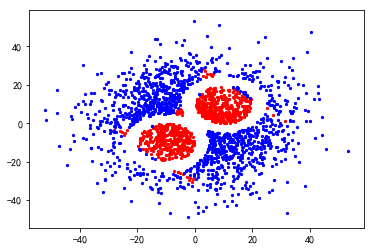} &  \includegraphics[width=0.2\linewidth]{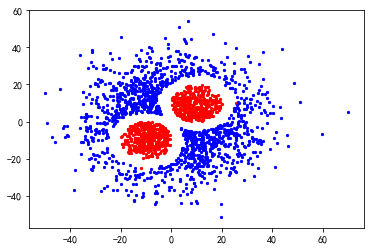}& \includegraphics[width=0.2\linewidth]{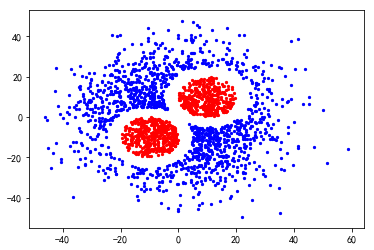}
	\end{tabular}
	\caption{Decision boundaries of classifying multiple unit balls obtained by training fully connected networks (top row) and AdderNets (bottom row) with different number of hidden units.}
	\label{fig:mulball}
\end{figure}

\begin{table*}[t]
	\centering
	\caption{Classification results on the CIFAR-10 and CIFAR-100 datasets.}
	\begin{tabular}{|c|c|c|c|c|c|c|}
		\hline
		\textbf{Model}&	\textbf{Method} & \textbf{\#Mul.} & \textbf{\#Add.} & \textbf{XNOR}  & \textbf{CIFAR-10} & \textbf{CIFAR-100} \\
		\hline	
		\hline	
		&BNN & 0 & 0.65G &  0.65G & 89.80\% & 67.24\% \\
		\cline{2-7}		
		&$\ell_1$-AddNN & 0 & 1.30G &0  & 93.72\% & 74.63\% \\
		\cline{2-7}		
		VGG-small&$\ell_2$-AddNN & 0.65G & 1.30G &0  & 94.20\% & 76.01\% \\
		\cline{2-7}
		&$\ell_1$-AddNN v2 & 0 & 1.30G &0  & 93.75\% & 74.85\% \\
		\cline{2-7}		
		&CNN &  0.65G& 0.65G& 0 &94.25\% & 75.96\% \\
		\hline
		\hline	
		&BNN & 0 &41.17M  &  41.17M   & 84.87\% & 54.14\% \\
		\cline{2-7}	
		&$\ell_1$-AddNN & 0 & 82.34M& 0  & 91.84\% & 67.60\% \\
		\cline{2-7}	
		ResNet-20&$\ell_2$-AddNN & 41.17M & 82.34M& 0  & 92.90\% & 68.71\% \\
		\cline{2-7}	
		&$\ell_1$-AddNN v2 & 0 & 82.34M& 0  & 92.31\% & 67.81\% \\
		\cline{2-7}		
		&CNN & 41.17M & 41.17M &  0 &92.93\% & 68.75\% \\
		\hline
		\hline	
		&BNN & 0 &69.12M  &  69.12M   & 86.74\% & 56.21\% \\
		\cline{2-7}	
		&$\ell_1$-AddNN & 0 & 138.24M& 0  & 93.01\% & 69.02\% \\
		\cline{2-7}	
		ResNet-32&$\ell_2$-AddNN &  69.12M  & 138.24M& 0  & 93.55\% & 70.48\% \\
		\cline{2-7}	
		&$\ell_1$-AddNN v2 & 0 & 138.24M& 0  & 93.10\% & 69.32\% \\
		\cline{2-7}		
		&CNN & 69.12M & 69.12M &  0 &93.59\% & 70.46\% \\
		\hline
	\end{tabular}
	\label{tab:cls}
\end{table*}

Since the classification task with multiple unit balls is more complex than that with a single ball, we use $n=3,5,10$ as the numbers of hidden units. As shown in figure~\ref{fig:mulball}, the traditional networks make some mistakes when $n=3,5$ and successfully classify most data points when $n=10$. For AdderNet, few data points are misclassified when $n=3$, and all data points are correctly classified when $n=5,10$. This shows that for certain problems where classes are centered, AdderNets can be superior to traditional neural networks.

\textbf{Linear Classification.} In this example, we consider a linear classification task using traditional networks and AdderNets to verify their capacity. The label $z_i$ for the input sample $(x_i,y_i)$ is calculated as
\begin{equation}
z_i=\left\{
\begin{aligned}
1 & , & x_i \times y_i \geq 0 , \\
0 & , & x_i \times y_i < 0,
\end{aligned}
\right.
\end{equation}
where $\{(x_i, y_i)\}$ generated from a two-dimensional normal distribution with mean of 0 and variance of 10.

\begin{figure}[h]
	\centering
	\begin{tabular}{cccc}
		Training Data  &  $n=1$  &  $n=2$ &  $n=3$ \\
		\includegraphics[width=0.2\linewidth]{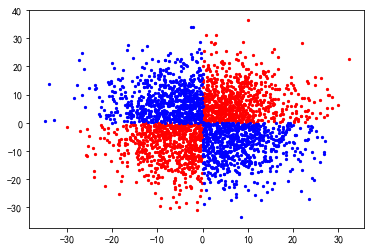}&  \includegraphics[width=0.2\linewidth]{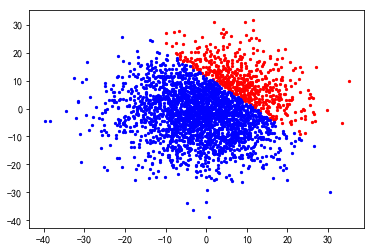} &  \includegraphics[width=0.2\linewidth]{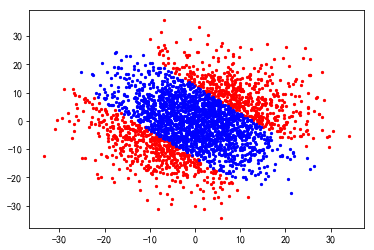}& \includegraphics[width=0.2\linewidth]{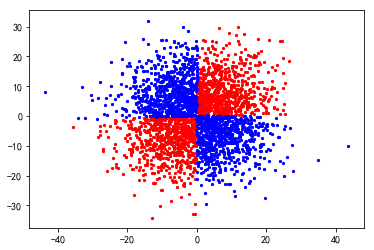}\\
		&   \includegraphics[width=0.2\linewidth]{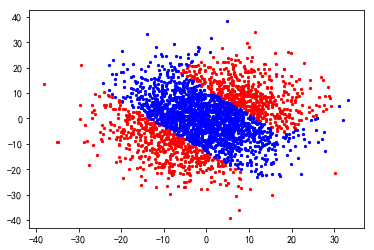} &  \includegraphics[width=0.2\linewidth]{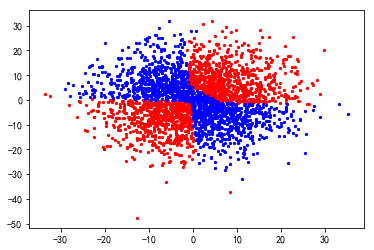}& \includegraphics[width=0.2\linewidth]{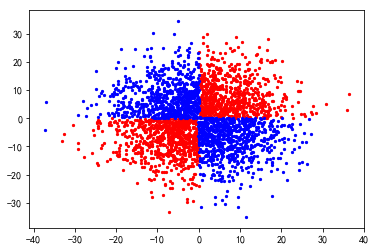}
	\end{tabular}
	\caption{Decision boundaries of classifying linear classification obtained by training fully connected networks (top row) and AdderNets (bottom row) with different number of hidden units.}
	\label{fig:line}
\end{figure}

Figure~\ref{fig:line} shows the classification results. When $n=1$ and $n=2$, neither AdderNets nor traditional networks performs well on this problem. When $n=3$, they can successfully classify all data points. The results empirically indicate that the adder neural network can also be regarded as a universal approximator and the approximation power increases with more hidden units.

\subsection{Experiments on MNIST}\label{sec:clas}

To illustrate the effectiveness of the proposed AdderNets, we first train a LeNet-5-BN~\cite{lenet} on the MNIST dataset. The images are resized to $32\times32$ and are pro-precessed following~\cite{lenet}. The networks are optimized using Nesterov Accelerated Gradient (NAG), and the weight decay and the momentum were set as $5\times10^{-4}$ and 0.9, respectively. We train the networks for 50 epochs using the cosine learning rate decay~\cite{loshchilov2016sgdr} with an initial learning rate 0.1. The batch size is set as 256. For the proposed AdderNets, we replace the convolutional filters in LeNet-5-BN with our adder filters. Note that the fully connected layer can be regarded as a convolutional layer, we also replace the multiplications in the fully connect layers with subtractions.

The convolutional neural network achieves a $99.4\%$ accuracy with $\sim$435K multiplications and $\sim$435K additions. By replacing the multiplications in convolution with additions, the proposed AdderNet achieves a 99.4\% accuracy, which is the same as that of CNNs, with $\sim$870K additions and almost no multiplication.
In fact, the theoretical latency of multiplications in CPUs is also larger than that of additions and subtractions. There is an instruction table~\footnote{www.agner.org/optimize/instruction\_tables.pdf} which lists the instruction latencies, throughputs and micro-operation breakdowns for Intel, AMD and VIA CPUs. For example, in VIA Nano 2000 series, the latency of float multiplication and addition is 4 and 2, respectively. The AdderNet using LeNet-5 model will have $\sim$1.7M latency while CNN will have $\sim$2.6M latency in this CPU. In conclusion, the AdderNet can achieve similar accuracy with CNN but have fewer computational cost and latency. Noted that CUDA and cuDNN optimized adder convolutions are not yet available, we do not compare the actual inference time.

\begin{table*}[t]
	\centering
	\caption{Classification results on the ImageNet datasets.}
	\begin{tabular}{|c|c|c|c|c|c|c|}
		\hline
		\textbf{Model} & \textbf{Method} & \textbf{\#Mul.} & \textbf{\#Add.} & \textbf{XNOR} & \textbf{Top-1 Acc.} & \textbf{Top-5 Acc.} \\
		\hline	
		\hline	
		& BNN & 0 & 1.8G &  1.8G &  51.2\% &  73.2\% \\
		\cline{2-7}	
		&$\ell_1$-AddNN & 0 & 3.6G &0  & 67.0\% & 87.6\% \\
		\cline{2-7}		
		ResNet-18& $\ell_2$-AddNN &1.8G & 3.6G &0  & 69.9\% & 89.1\% \\
		\cline{2-7}	
		&$\ell_1$-AddNN v2& 0 & 3.6G &0  & 69.1\% & 88.4\% \\
		\cline{2-7}	
		&CNN &  1.8G& 1.8G&  0  & 69.8\% & 89.1\% \\
		\hline
		\hline	
		& BNN & 0 & 3.9G & 3.9G & 55.8\% &  78.4\% \\
		\cline{2-7}	
		&$\ell_1$-AddNN & 0 & 7.7G &0  & 74.9\% & 91.7\% \\
		\cline{2-7}
		ResNet-50& $\ell_2$-AddNN & 3.9G & 7.7G &0  & 76.1\% & 92.8\% \\
		\cline{2-7}		
		&$\ell_1$-AddNN v2 & 0 & 7.7G &0  & 75.7\% & 92.3\% \\
		\cline{2-7}
		&CNN &  3.9G& 3.9G&  0  & 76.2\% & 92.9\% \\
		\hline
	\end{tabular}
	\label{tab:ImageNet}
\end{table*}

\subsection{Experiments on CIFAR}

We then evaluate our method on the CIFAR dataset, which consist of $32\times32$ pixel RGB color images. Since the binary networks~\cite{zhou2016dorefa} can use the XNOR operations to replace multiplications, we also compare the results of binary neural networks (BNNs). We use the same data augmentation and pro-precessing in He~\etal~\cite{he2016deep} for training and testing. Following Zhou~\etal~\cite{zhou2016dorefa}, the learning rate is set to 0.1 in the beginning and then follows a polynomial learning rate schedule. The models are trained for 800 epochs with a 256 batch size. We follow the general setting in binary networks to set the first and last layers as full-precision convolutional layers. In AdderNets, we use the same setting for a fair comparison. The hyper-parameter $\eta$ is set to 0.1 following the experiments on the MNIST dataset.

The classification results are reported in Table~\ref{tab:cls}. Since computational cost in batch normalization layer, the first layer and the last layer are significantly less than other layers, we omit these layers when counting FLOPs. We first evaluate the VGG-small model~\cite{cai2017deep} in the CIFAR-10 and CIFAR-100 dataset. As a result, the AdderNets achieve nearly the same results with CNNs with no multiplication. Although the model size of BNN is much smaller than those of AdderNet and CNN, its accuracies are much lower. We then turn to the widely used ResNet models (ResNet-20 and ResNet-32) to further investigate the performance of different networks. As for the ResNet-20, The convolutional neural networks achieve the highest accuracy but with a large number of multiplications (41.17M). The proposed AdderNets achieve a 92.31\% accuracy in CIFAR-10 and a 67.81\% accuracy in CIFAR-100 without multiplications, which is comparable with CNNs. In contrast, the BNNs only achieve 84.87\% and 54.14\% accuracies in CIFAR-10 and CIFAR-100. The results in ResNet-32 also suggest that the proposed AdderNets can achieve similar results with conventional CNNs.

\subsection{Experiments on ImageNet}

We next conduct experiments on the ImageNet dataset~\cite{krizhevsky2012imagenet}, which consist of $224\times224$ pixel RGB color images. We use ResNet-18 model to evaluate the proposed AdderNets follow the same data augmentation and pro-precessing in He~\etal~\cite{he2016deep}. We train the AdderNets for 300 epochs utilizing the cosine learning rate decay~\cite{loshchilov2016sgdr}. These networks are optimized using Nesterov Accelerated Gradient (NAG), and the weight decay and the momentum are set as $10^{-4}$ and 0.9, respectively. The batch size is set as 256 and the hyper-parameter in AdderNets is the same as that in CIFAR experiments. 

Table~\ref{tab:ImageNet} shows the classification results on the ImageNet dataset by exploiting different nerual networks. The convolutional neural network achieves a 69.8\% top-1 accuracy and an 89.1\% top-5 accuracy in ResNet-18. However, there are 1.8G multiplications in this model, which bring enormous computational complexity. Since the addition operation has smaller computational cost than multiplication, we propose AdderNets to replace the multiplications in CNNs with subtractions. As a result, our AdderNet achieve a 69.1\% top-1 accuracy and an 88.4\% top-5 accuracy in ResNet-18, which demonstrate the adder filters can extract useful information from images. Rastegari~\etal~\cite{rastegari2016xnor} proposed the XNOR-net to replace the multiplications in neural networks with XNOR operations. Although the BNN can achieve high speed-up and compression ratio, it achieves only a 51.2\% top-1 accuracy and a 73.2\% top-5 accuracy in ResNet-18, which is much lower than the proposed AdderNet. We then conduct experiments on a deeper architecture (ResNet-50). The BNN could only achieve a 55.8\% top-1 accuracy and a 78.4\% top-5 accuracy using ResNet-50. In contrast, the proposed AdderNets can achieve a 75.7\% top-1 accuracy and a 92.3\% top-5 accuracy, which is closed to that of CNN (76.2\% top-1 accuracy and 92.9\% top-5 accuracy).

\subsection{Visualization Results}

\begin{figure}[h]
	\centering
	\includegraphics[width=0.9\linewidth]{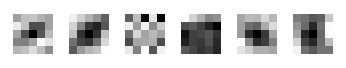} \\
	(a) Visualization of filters of AdderNets\\
	\includegraphics[width=0.9\linewidth]{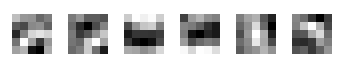} \\
	(b)  Visualization of filters of CNNs  
	\caption{Visualization of filters in the first layer of LeNet-5-BN on the MNIST dataset. Both of them can extract useful features for image classification.}
	\label{Fig:visualfilters}
\end{figure}

\textbf{Visualization on filters.} We visualize the filters of the LeNet-5-BN network in Figure~\ref{Fig:visualfilters}. Although the AdderNets and CNNs utilize different distance metrics, filters of the proposed adder networks (see Figure~\ref{Fig:visualfilters} (a)) still share some similar patterns with convolution filters (see Figure~\ref{Fig:visualfilters} (b)). The visualization experiments further demonstrate that the filters of AdderNets can effectively extract useful information from the input images and features. 

\begin{table*}[t]
	\begin{center}
		\caption{Effectiveness of different training strategy in AdderNets.}
		\label{table:abl}
		\begin{tabular}{|c|c|c|c|c|c|c|c|c|}
			\hline
			full-precision gradient & & \checkmark & &&\checkmark & \checkmark&& \checkmark \\
			\hline
			$\ell_2$ to $\ell_1$ training & & & \checkmark & &\checkmark& &\checkmark& \checkmark \\
			\hline
			Adaptive learning rate &  & &  & \checkmark & &\checkmark&\checkmark& \checkmark \\
			\hline
			\textbf{Top 1 accuracy} & 80.77\% &85.34\%& 86.21\% & 90.14\% & 87.61\% &91.83\% &91.80\% & 92.32\%\\
			\hline
		\end{tabular}
	\end{center}
\end{table*}

\begin{figure*}[t]
	\centering
		\begin{tabular}{cc}
		\includegraphics[width=0.46\linewidth]{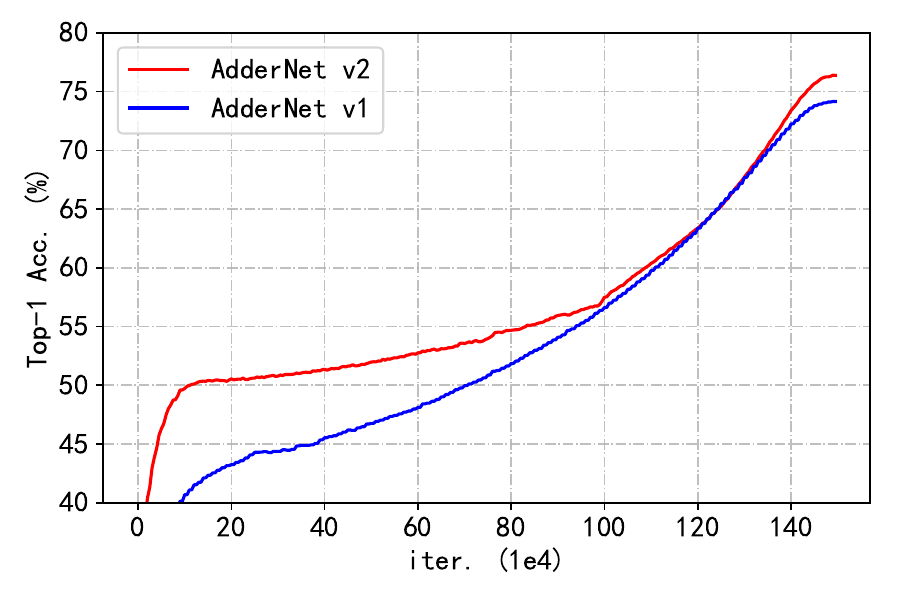} &
		\quad \includegraphics[width=0.46\linewidth]{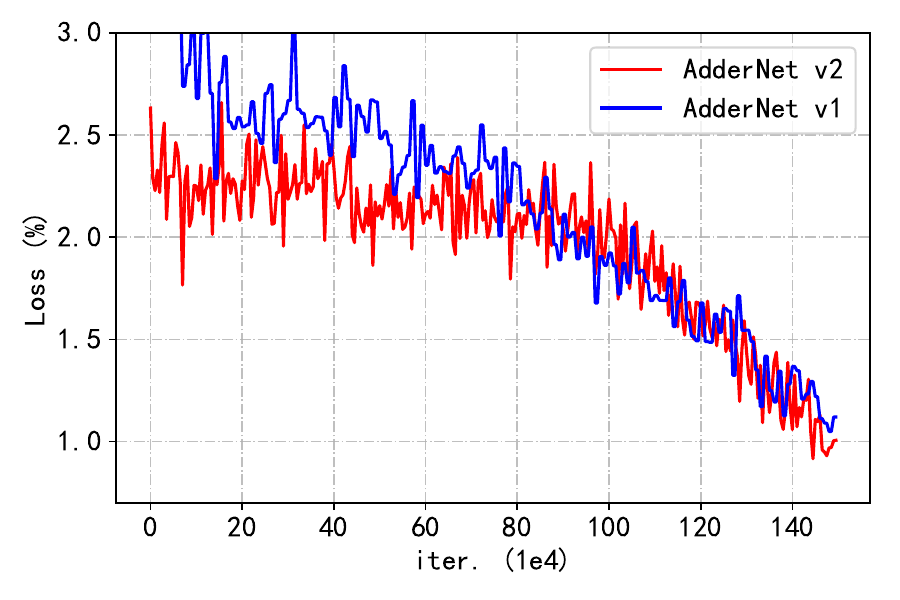} \\
		(a) Accuracy   &(b)  Loss  
	\end{tabular}
	\caption{Learning curve of AdderNet v1 and v2. By using the proposed $\ell_2$ to $\ell_1$ training strategy, AdderNet achieves higher accuracy (75.7\%) with smaller loss on ResNet-50 architecture.}
	\label{Fig:abl}
\end{figure*}

\textbf{Visualization on features.} The AdderNets utilize the $\ell_1$-distance to measure the relationship between filters and input features instead of cross correlation in CNNs. Therefore, it is important to further investigate the difference of the feature space in AdderNets and CNNs. We train a LeNet++ on the MNIST dataset following~\cite{centerloss}, which has six convolutional layers and a fully-connected layer for extracting powerful 3D features. Numbers of neurons in each convolutional layer are 32, 32, 64, 64, 128, 128, and 2, respectively.
For the proposed AdderNets, the last fully connected layers are replaced with the proposed add filters. 

The visualization results are shown in Figure~\ref{Fig:visualfea}. The convolutional neural network calculates the cross correlation between filters and inputs. If filters and inputs are approximately normalized, convolution operation is then equivalent to calculate cosine distance between two vectors. That is probably the reason that features in different classes are divided by their angles in Figure~\ref{Fig:visualfea}. In contrast, AdderNets utilize the $\ell_1$-norm to distinguish different classes. Thus, features tend to be clustered towards different class centers. The visualization results demonstrate that the proposed AdderNets could have the similar discrimination ability to classify images as CNNs.

\begin{figure}[h]
	\centering
	\includegraphics[width=1.0\linewidth]{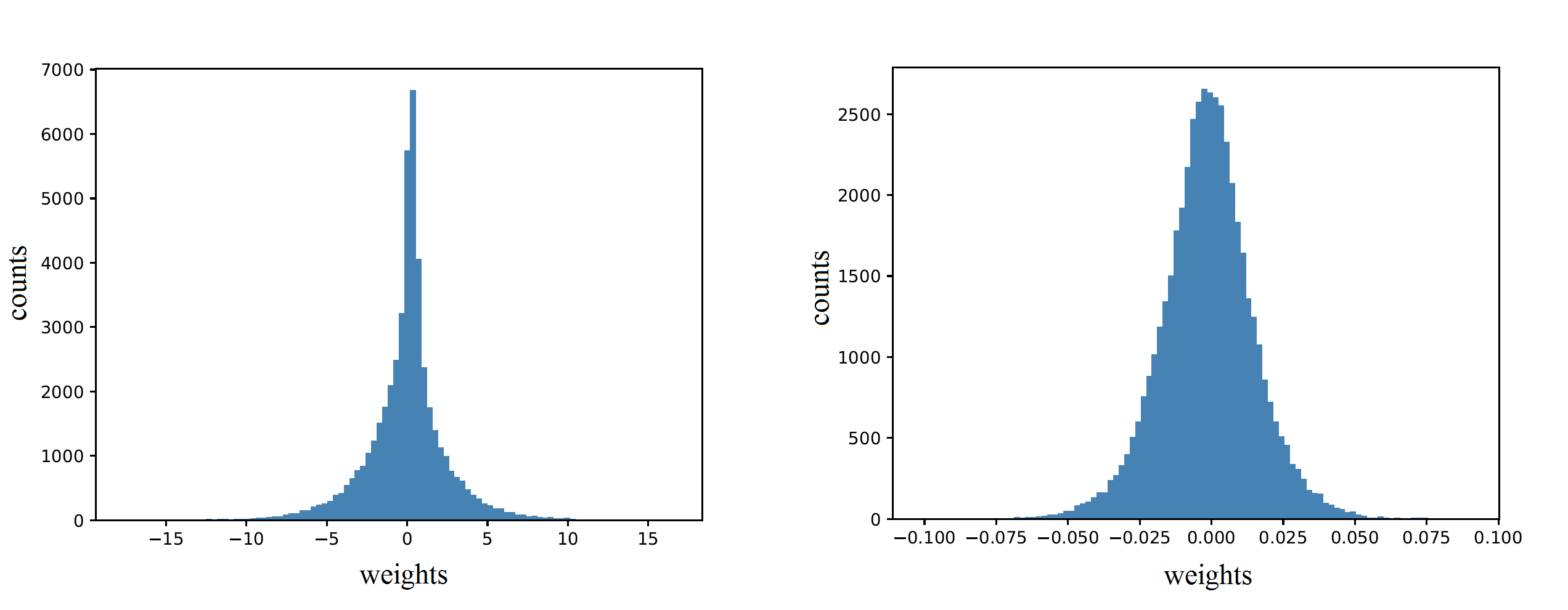}
	\caption{Histograms over the weights with AdderNet (left) and CNN (right). The weights of AdderNets follow Laplace distribution while those of CNNs follow Gaussian distribution. }
	\label{Fig:his}
\end{figure}

\textbf{Visualization on distribution of weights.} We then visualize the distribution of weights for the 3rd convolution layer on LeNet-5-BN. As shown in Figure~\ref{Fig:his}, the distribution of weights with AdderNets is close to a Laplace distribution while that with CNNs looks more like a Gaussian distribution. In fact, the prior distribution of $\ell_1$-norm is Laplace distribution~\cite{stigler1986history} and that of $\ell_2$-norm is Gaussian distribution~\cite{rennie2003l2}.

\subsection{Ablation Study}

We propose to use an $\ell_2$ to $\ell_1$ training strategy to update the filters in our adder filters and design an adaptive learning rate scaling for deal with different layers in AdderNets. It is essential to evaluate the effectiveness of these components. As shown in Table~\ref{table:abl}, applying the proposed techniques can successfully improve the performance of AdderNets. We further investigate the influence of the $\ell_2$ to $\ell_1$ training strategy in Figure~\ref{Fig:abl}. We find that AdderNet v2 can be trained more easily with the help of the $\ell_2$ norm.

\begin{table}[h]
	\centering
	\caption{The impact of parameter $\eta$ using LeNet-5-BN on the MNIST dataset.}
	\begin{tabular}{|c|c|c|c|c|c|}
		\hline
		$\eta$ & 1 & 0.5 & 0.2 & 0.1 & 0.05 \\
		\hline	
		\hline	
		Acc. (\%) & 99.28 &99.35&99.40& 99.35 & 99.30 \\
		\hline
	\end{tabular}
	\label{tab:impact}
\end{table}

\textbf{Impact of parameters.} As discussed above, the proposed adaptive learning rate scaling has a hyper-parameter: $\eta$. We then test its impact on the accuracy of the student network by conducting the experiments on the MNIST dataset. We use LeNet-5-BN as the backbone of AdderNet. Other experimental settings are same as mentioned in Sec.~\ref{sec:clas}. It can be seen from Table~\ref{tab:impact} that the AdderNets trained utilizing the adaptive learning rate scaling achieves the highest accuracy (99.40\%) when $\eta$ = 0.2. Based on the above analysis, we keep the setting of hyper-parameters for the proposed method.

\begin{table}[h]
	\centering
	\caption{The impact of number of epoch when $p=1$.}
	\begin{tabular}{|c|c|c|c|c|c|}
		\hline
		epoch & 300 & 400 & 500 & 600 & 700 \\
		\hline	
		\hline	
		Acc. (\%) & 91.91 &92.00&92.11& 92.31 & 91.76 \\
		\hline
	\end{tabular}
	\label{tab:decay}
\end{table}

\textbf{Decay of $p$.} We use the $\ell_p$ norm to training AdderNet, where the $p$ is linearly decayed from 2 to 1. We analyze the impact of number of epoch when $p$ is decayed to 1 on the CIFAR-10 dataset. The network is trained for 800 epochs. As shown in Table~\ref{tab:decay}. If the number of decayed epoch is too large (\eg 700 epoch), the network achieves only 91.76\% accuracy, since the AdderNet is not fully trained in $\ell_1$ norm. If the number of decayed epoch is too small (\eg 300 epoch), the network is rarely trained with $\ell_p$ norm ($p>1$), therefore the $\ell_p$ cannot help the training of $\ell_1$ norm. The network achieves the highest performance with 600 decayed epochs, and we keep this setting of hyper-parameters in other experiments.

\section{Conclusion}
The role of classical convolutions used in deep CNNs is to measure the similarity between features and filters, and we are motivated to replace convolutions with more efficient similarity measure. We investigate the feasibility of replacing multiplications by additions in this work. An AdderNet is explored to effectively use addition to build deep neural networks with low computational costs. This kind of networks calculate the $\ell_1$-norm distance between features and filters. We proved that AdderNets are universal approximators, which is analogous to the universal approximation results for two-layer traditional network networks~\cite{hornik1989multilayer} and width-bounded deep networks~\cite{lu2017expressive}, etc. The universal approximation theorem gives us assurance that AdderNets can solve an arbitrary problem. We further provided empirical results showing that AdderNets can indeed approximate a complex decision boundary with a sufficient number of hidden units or depths. Corresponding optimization method is developed by using $\ell_p$-norm. Experiments conducted on benchmark datasets show that AdderNets can well approximate the performance of CNNs with the same architectures, which could have a huge impact on future hardware design. Visualization results also demonstrate that the adder filters are promising to replace original convolution filters for computer vision tasks.

%

\bibliography{ref}
\bibliographystyle{IEEEtran}

%
\begin{IEEEbiography}
	[{\includegraphics[width=1in,height=1.25in,clip,keepaspectratio]{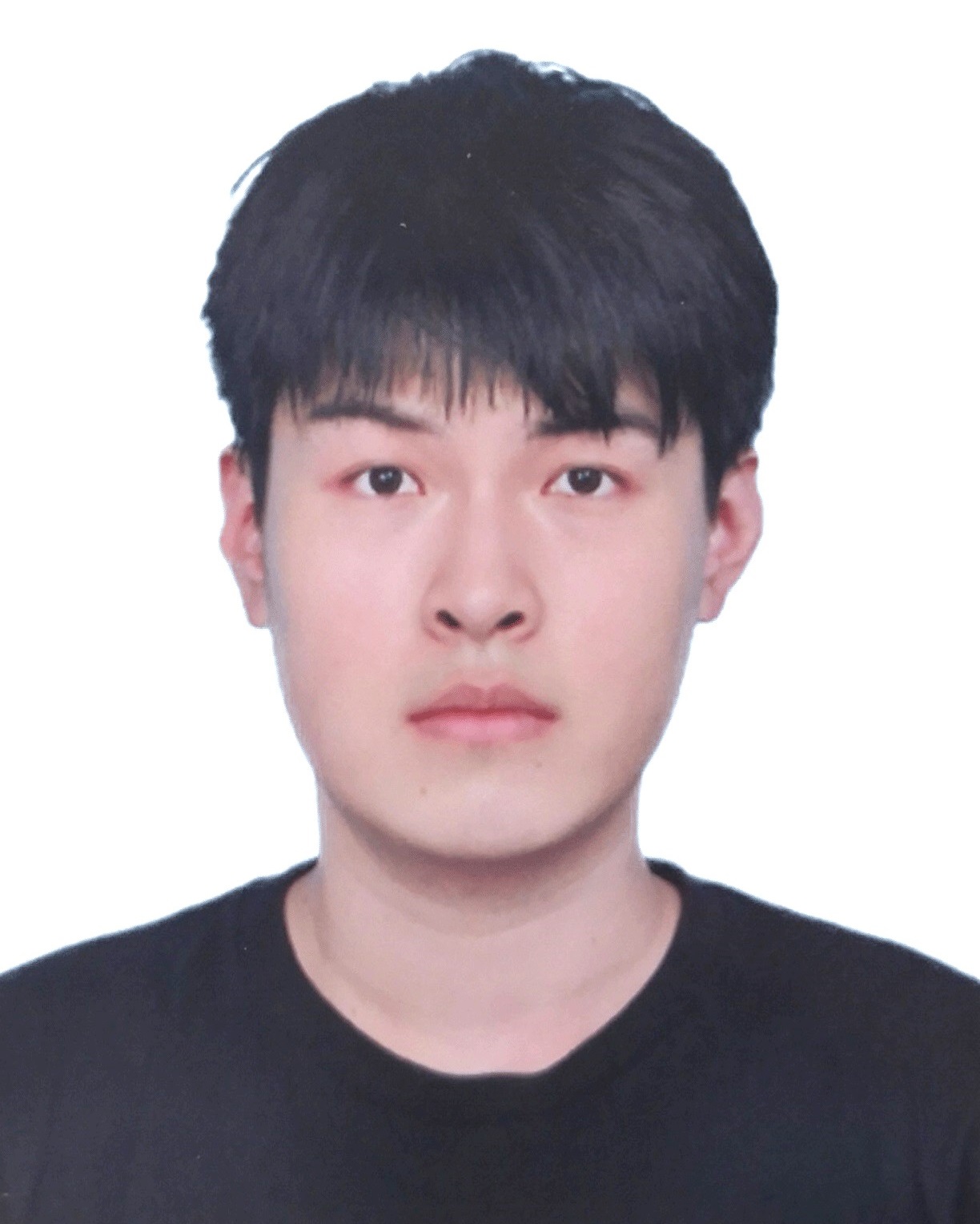}}]{Hanting Chen} received a B.E degree from Tongji University, China. Currently, he is a Ph.D. candidate with the Key Laboratory of Machine Perception (Ministry of Education) at Peking University. His research interests lie primarily in machine learning and computer vision.
\end{IEEEbiography}

\begin{IEEEbiography}
	[{\includegraphics[width=1in,height=1.25in,clip,keepaspectratio]{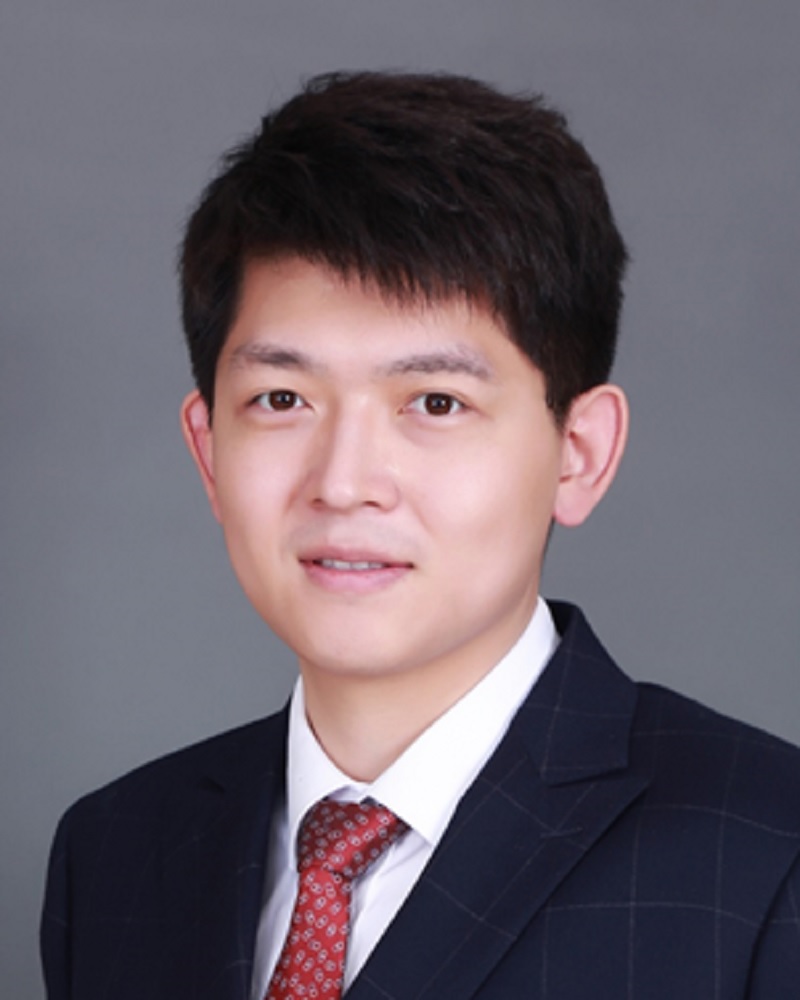}}]{Yunhe Wang} is currently a Senior Researcher at Noah's Ark Lab, Huawei Technologies. He received a Ph.D. degree from Peking University, China. His research interests mainly include machine learning, computer vision, and efficient deep learning. He has published over 50 papers in prestigious journals and conferences. Many of them are widely applied in industrial products, and received important awards. He leads the ``Adder Neural Networks'' project, which presents a new computing paradigm for artificial intelligence with significantly lower hardware costs. He regularly serves as PC/senior PC member for top conferences, \emph{e.g.}, NeurIPS, ICML, ICLR, CVPR, ICCV, IJCAI, and AAAI.
\end{IEEEbiography}

\begin{IEEEbiography}
	[{\includegraphics[width=1in,height=1.25in,clip,keepaspectratio]{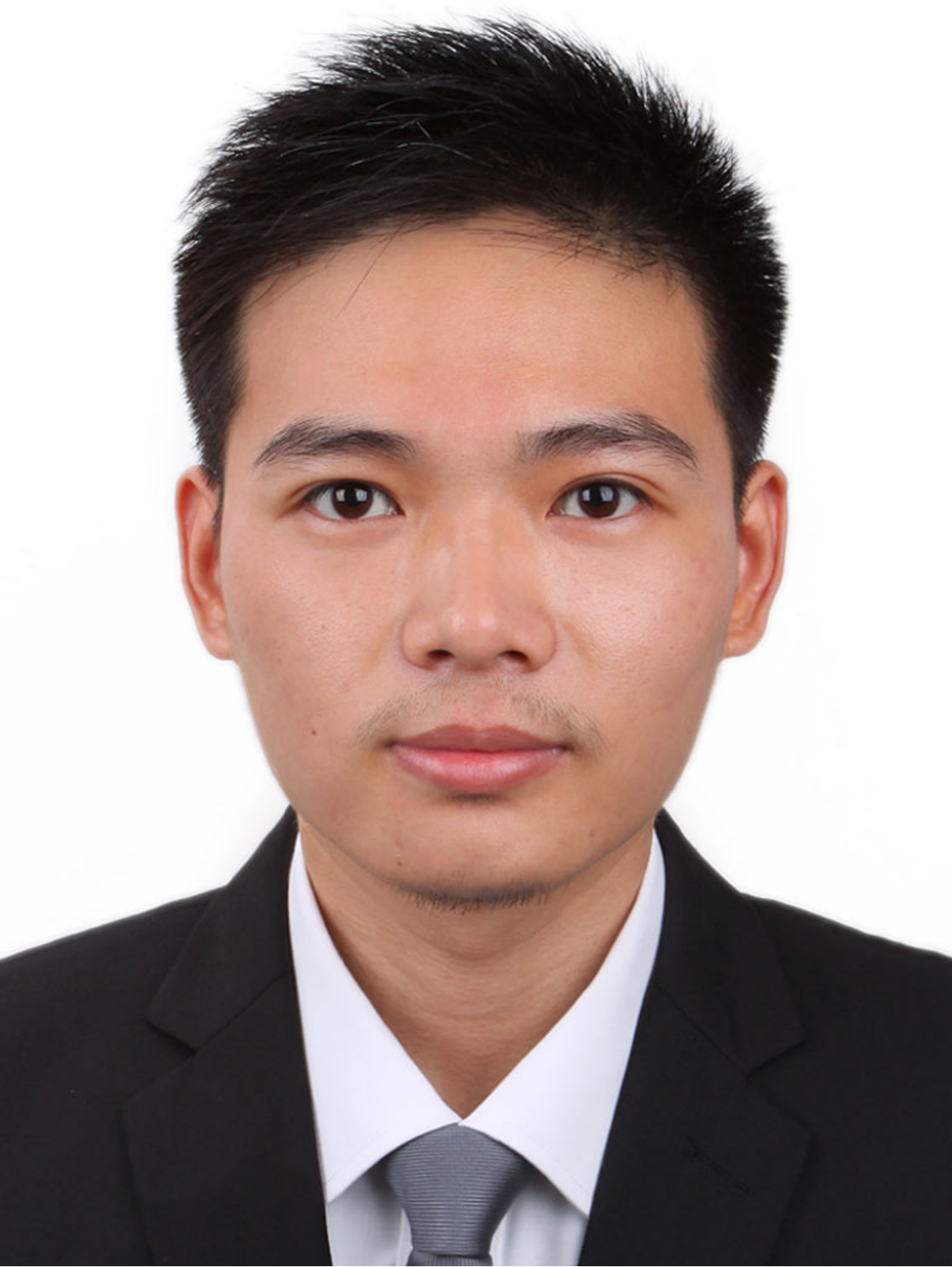}}]{Chang Xu} is Senior Lecturer and ARC DECRA Fellow at the School of Computer Science, University of Sydney. He received the Ph.D. degree from Peking University, China. His research interests lie in machine learning algorithms and related applications in computer vision. He has published over 100 papers in prestigious journals and top tier conferences. He has received several paper awards, including Distinguished Paper Award in IJCAI 2018. He regularly severed as the PC member or senior PC member for many conferences, \emph{e.g.}, NeurIPS, ICML, ICLR, CVPR, ICCV, IJCAI and AAAI. He has been recognized as Top Ten Distinguished Senior PC Member in IJCAI 2017.
\end{IEEEbiography}

\begin{IEEEbiography}
	[{\includegraphics[width=1in,height=1.25in,clip,keepaspectratio]{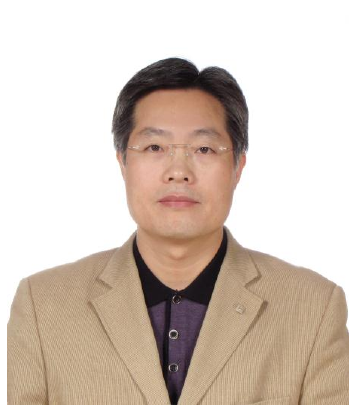}}]{Chao Xu} received the B.E. degree from Tsinghua University in 1988, the M.S. degree from University of Science and Technology of China in 1991 and the Ph.D degree from Institute of Electronics, Chinese Academy of Sciences in 1997. Between 1991 and 1994 he was employed as an assistant professor by University of Science and Technology of China. Since 1997 Dr. Xu has been with School of EECS at Peking University where he is currently a Professor. His research interests are in image and video coding, processing and understanding. He has authored or co-authored more than 80 publications and 5 patents in these fields.
\end{IEEEbiography}

\begin{IEEEbiography}
	[{\includegraphics[width=1in,height=1.25in,clip,keepaspectratio]{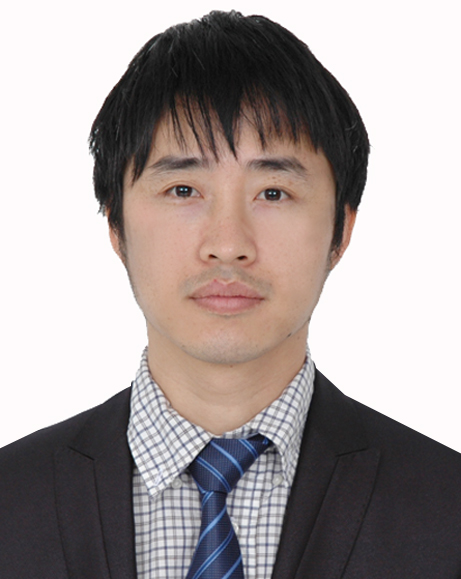}}]{Chunjing Xu} received the B.E. degree from Wuhan University, China, the M.S. degree from Peking University, China, and the Ph.D. degree from the Chinese University of Hong Kong, Hong Kong, China. He is currently a principal researcher with Huawei Noah's Ark Lab. His research interests lie primarily in deep learning, and computer vision.
\end{IEEEbiography}

\begin{IEEEbiography}
	[{\includegraphics[width=1in,height=1.25in,clip,keepaspectratio]{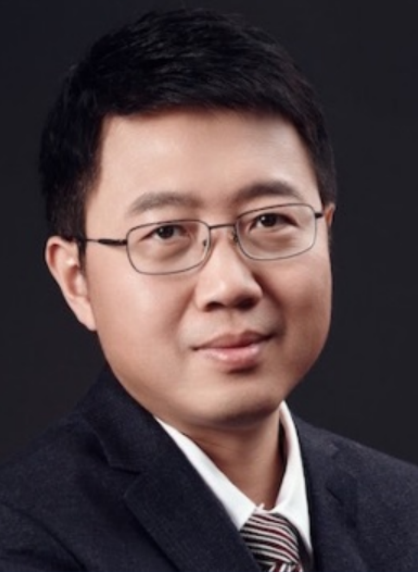}}]{Tong Zhang} is a professor of Computer Science and Mathematics at the Hong Kong University of Science and Technology. Previously, he was a professor at Rutgers university, and worked at IBM, Yahoo, Baidu, and Tencent. He received a B.A. in mathematics and computer science from Cornell University and a Ph.D. in Computer Science from Stanford University. His research interests include machine learning algorithms and theory, statistical methods for big data and their applications. He is a fellow of ASA and IMS, and he has been in the editorial boards of leading machine learning journals and program committees of top machine learning conferences.

\end{IEEEbiography}




\end{document}